\documentclass[sigconf]{acmart}

\AtBeginDocument{%
  }

\copyrightyear{2025}  
\acmYear{2025}  
\setcopyright{acmlicensed}  
\acmConference[KDD '25]{Proceedings of the 31st ACM SIGKDD Conference on Knowledge Discovery and Data Mining V.1}{August 3--7, 2025}{Toronto, ON, Canada.}  
\acmBooktitle{Proceedings of the 31st ACM SIGKDD Conference on Knowledge Discovery and Data Mining V.1 (KDD '25), August 3--7, 2025, Toronto, ON, Canada}  
\acmISBN{979-8-4007-1245-6/25/08}  
\acmDOI{xxxxxxxxx}
\usepackage{algorithm}
\usepackage{algorithmicx}
\usepackage{algpseudocode}
\usepackage{epstopdf}
\usepackage{multirow}
\usepackage{enumitem}

\usepackage{makecell}
\usepackage{color, colortbl}
\definecolor{Gray}{gray}{0.9}

\begin{document}


\title{Fast Causal Discovery by Approximate Kernel-based Generalized Score Functions with Linear Computational Complexity}

\author{Yixin Ren}
\email{yxren21@m.fudan.edu.cn}
\affiliation{%
\institution{Fudan University}
\state{Shanghai}
\country{China}}
\authornotemark[1]

\author{Haocheng Zhang}
\email{hczhang23@m.fudan.edu.cn}
\affiliation{%
\institution{Fudan University}
\state{Shanghai}
\country{China}}
\authornote{Co-first authors.}

\author{Yewei Xia}
\email{ywxia23@.m.fudan.edu.cn}
\affiliation{%
\institution{Fudan University}
\state{Shanghai}
\country{China}}

\author{Hao Zhang}
\email{h.zhang10@siat.ac.cn}
\affiliation{%
    \institution{SIAT, Chinese Academy of Sciences}
\state{Shenzhen}
\country{China}}
\authornotemark[2]

\author{Jihong Guan}
\email{jhguan@tongji.edu.cn}
\affiliation{%
\institution{Tongji University}
\state{Shanghai}
\country{China}}

\author{Shuigeng Zhou}
\email{sgzhou@fudan.edu.cn}
\affiliation{%
\institution{Fudan University}
\state{Shanghai}
\country{China}}
\authornote{Corresponding authors.}

\renewcommand{\shortauthors}{Yixin Ren et al.}
\begin{abstract}
Score-based causal discovery methods can effectively identify causal relationships by evaluating candidate graphs and selecting the one with the highest score. One popular class of scores is kernel-based generalized score functions, which can adapt to a wide range of scenarios and work well in practice because they circumvent assumptions about causal mechanisms and data distributions. Despite these advantages, kernel-based generalized score functions pose serious computational challenges in time and space, with a time complexity of $\mathcal{O}(n^3)$ and a memory complexity of $\mathcal{O}(n^2)$, where $n$ is the sample size. In this paper, we propose an approximate kernel-based generalized score function with $\mathcal{O}(n)$ time and space complexities by using low-rank technique and designing a set of rules to handle the complex composite matrix operations required to calculate the score, as well as developing sampling algorithms for different data types to benefit the handling of diverse data types efficiently. Our extensive causal discovery experiments on both synthetic and real-world data demonstrate that compared to the state-of-the-art method, our method can not only significantly reduce computational costs, but also achieve comparable accuracy, especially for large datasets.
\end{abstract}




\maketitle

\section{Introduction}
The capacity to understand causality is a key aspect of human-level intelligence~\cite{pearl2018theoretical}. Consequently, automatic identification of causal relationships by data analysis is essential for achieving strong AI~\cite{pearl2018book}. Effective knowledge transfer and generalization by machines also depend on their abilities to comprehend the causal connections within data~\cite{scholkopf2021toward, zhang2024causal}. Despite its importance, identifying causal relationships among random variables remains a long-standing challenge in artificial intelligence. While randomized experiments are considered the gold standard for causal analysis, their applications are often limited by ethical and time-cost constraints~\cite{spirtes2016causal}. Therefore, there is a growing interest in developing computational methods to infer causal structures from observational data. Recently, score-based methods have emerged as a promising tool, facilitating the rapid progress in causal discovery~\cite{zhang2018learning,vowels2022d,zheng2024causal}. 

Score-based methods evaluate the quality of candidate causal graphs using predefined score functions. By combining these scores with various search strategies, such as combinatorial search-based methods~\cite{chickering2002optimal, chickering2020statistically, tsamardinos2006max, yuan2013learning} and continuous optimization-based methods~\cite{zheng2018dags, ng2020role, yu2019dag, wei2020dags, zheng2020learning, zhang2022truncated, sun2021nts, xia2023causal, ng2024structure}, these algorithms can output the optimal causal graph or equivalent class. Up to now, various score functions~\cite{hyvarinen2013pairwise, buhlmann2014cam} have been proposed for this purpose, but many of them rely on strong assumptions about the underlying data generation models. For example, the BIC score~\cite{schwarz1978estimating} and the BGe score~\cite{geiger1994learning} assume linear-Gaussian models while the BDeu/BDe score~\cite{buntine1991theory, heckerman1995learning} is only applicable to discrete data. These assumptions restrict their use in real-world scenarios, and the model misspecification problem may give rise to misleading results. 

To handle a wider variety of causal relationships and data distributions, kernel-based scores~\citep{bach2002learning,huang2018generalized, wangoptimal} were proposed. These scores utilize the characterization of kernels~\citep{scholkopf2002learning}, making them highly adaptable to various datasets. Among them, the generalized score functions~\cite{huang2018generalized} are particularly effective in practice. These scores are based on the characterization of general (conditional) independence in reproducing kernel Hilbert spaces (RKHSs)~\cite{zhang2012kernel} and serve as model selection criteria for the regression problems involving kernel feature maps in RKHSs. Additionally, to mitigate the issue of overfitting in finite sample estimates, two approaches can be exploited: one uses cross-validated likelihood~\cite{huang2018generalized}, and the other focuses on maximizing the marginal likelihood~\cite{huang2018generalized, wangoptimal}. Both methods are effective, with the marginal likelihood approach requiring an additional optimization process to obtain the optimal parameters. Despite their advantages, these generalized score functions pose serious computational challenges. Calculating them involves matrix multiplications, inverses, and determinants related to the kernel matrices, and necessitates considerable storage for the kernel matrix terms. As a result, they have a time complexity of $\mathcal{O}(n^3)$ and a memory complexity of $\mathcal{O}(n^2)$, where $n$ is the sample size. 

The computation bottleneck above constitutes a fundamental problem with kernel-related tasks~\cite{wang2020survey} such as hypothesis testing~\cite{strobl2019approximate, zhang2018large, renefficiently, ren2024learning} and Gaussian process~\cite{williams1995gaussian, gardner2018gpytorch, liu2020gaussian} in several large-scale learning scenarios. From the calculation perspective, low-rank techniques~\cite{si2017memory} (e.g. the Nystr{\"o}m method~\cite{williams2000using} and random feature~\cite{rahimi2007random}) can be used to generate matrix approximations. Unlike random feature methods that use data-independent sampling, the Nystr{\"o}m method employs data-dependent sampling, often yields better results~\citep{yang2012nystrom}. Inspired by this point, in this paper we try to explore the Nystr{\"o}m method to accelerate the computation of generalized score functions. However, incorporating low-rank approximation into our task requires further exploration due to the complex composite matrix multiplication and inverse operations involved in score computation. We address this by designing a set of rules to transform the score components into a specific form that can effectively simplify the computation. Additionally, for different data types, we develop more efficient sampling algorithms tailored to different types of data, further boosting the performance of the approximation algorithm. These measures result in linear computation complexity in both time and space. 

\noindent \textbf{Contributions.} In summary, our contributions are as follows:
\begin{itemize}[leftmargin=4mm]
\item We propose an approximate kernel-based generalized score function by low-rank approximation and a set of rules to handle the complex composite matrix operations involved in score calculation. This results in $\mathcal{O}(n)$ time and space complexities.
\item We design efficient sampling algorithms, tailored to different data types, to further boost the performance of the approximation algorithm. Such a solution enables the score function to handle diverse data types efficiently.
\item We conduct extensive experiments on both synthetic and real-world data, which show that our method significantly reduces computational costs compared to the SOTA method, while achieving comparable accuracy, especially for large datasets.
\end{itemize}

\noindent \textbf{Outline.} The rest of the paper is organized as follows: Sec.~\ref{sec:preliminaries} reviews the model selection problem to capture conditional independence in RKHSs. Sec.~\ref{sec:generalizedscore} presents the generalized score functions for model selection and describes the computational challenges. Sec.~\ref{sec:generalizedlowrank} introduces the low-rank method for obtaining approximate kernel. Sec.~\ref{sec: scorefunctionswithapproximatekernel} proposes the score function with approximate kernel. Sec.~\ref{sec:causalsearch} describes the causal search procedure with the proposed score. Sec.~\ref{sec:experimental} is performance evaluation. Sec.~\ref{sec:conclusion} concludes the paper. 

\section{Preliminaries}\label{sec:preliminaries}
\textbf{Notions.} We denote $X,Y,Z$ the sets of random variables with domains $\mathcal{X},\mathcal{Y},\mathcal{Z}$ respectively. Consider a continuous feature mapping  $\phi_{\mathcal{X}}: \mathcal{X}\mapsto \mathcal{H}_{\mathcal{X}}$ with the corresponding measureable positive definite kernel $k_{\mathcal{X}}:=\langle \phi_{\mathcal{X}},\phi_{\mathcal{X}}\rangle$, where $\mathcal{H}_{\mathcal{X}}$ is the corresponding RKHS. 
We denote the probability distribution of $X$ as $\mathbb{P}_{X}$, and the corresponding square integrable spaces as $\mathcal{L}^2_{X}$. We assume that $\mathcal{H}_{\mathcal{X}}\subset \mathcal{L}^2_{X}$. The notions for $Y,Z$ are defined by analogy. Also, we denote $\mathbb{P}_{XZ}$ as the joint distribution of $(X,Z)$ and $\mathbb{P}_{X|Z}$ as the conditional distribution given $Z$. When $\mathbb{P}_{XY|Z} = \mathbb{P}_{X|Z}\mathbb{P}_{Y|Z}$, i.e., $X$ and $Y$ are conditional independent (CI) given $Z$, we use the notion $X\perp \!\!\! \perp Y|Z$ for simplification. Next, we introduce the \textbf{CI characterization in the  RKHSs}~\cite{fukumizu2004dimensionality,fukumizu2007kernel}. We first have the following definition:
\begin{definition}[~\cite{fukumizu2004dimensionality,fukumizu2007kernel}] For the random vector $(X,Y)$ on $\mathcal{X}\times \mathcal{Y}$, we define the cross-covariance operator from $\mathcal{H}_\mathcal{Y}$ to $\mathcal{H}_\mathcal{X}$ as 
\begin{equation}
\langle f,\Sigma_{XY}g\rangle:= \mathbb{E}_{XY}[f(X)g(Y)] - \mathbb{E}_{X}[f(X)]\mathbb{E}_{Y}[g(Y)]
\end{equation}
for all $f\in \mathcal{H}_\mathcal{X}$ and $g\in \mathcal{H}_\mathcal{Y}$. Then the partial cross-covariance operator of $(X,Y)$ given $Z$ is given by $\Sigma_{XY| Z}:= \Sigma_{XY}-\Sigma_{XZ}\Sigma_{ZZ}^{-1}\Sigma_{ZY}.$
\end{definition}
\noindent If the kernels are characteristic~\cite{sriperumbudur2010hilbert,fukumizu2008characteristic} and bounded, the partial cross-covariance operator is related to CI by the following lemma:  
\begin{lemma}[~\cite{fukumizu2004dimensionality,fukumizu2008characteristic}] Let $\ddot{Z}:=(Y,Z)$ and $\otimes$ be the direct product symbol of spaces. It is assumed that $\mathbb{E}_{X|Z} [g(X)|Z=\cdot]\in 
\mathcal{H}_Z$ and
$\mathbb{E}_{X|\ddot{Y}}[g(X)|\ddot{Z}=\cdot]\in \mathcal{H}_Y\otimes \mathcal{H}_Z $ for all $g\in \mathcal{H}_{\mathcal{X}}$ and further $\mathcal{H}_{\mathcal{X}}$ is
probability-determining~\cite{fukumizu2004dimensionality}, then $\Sigma_{XX| Z} \geq \Sigma_{XX| \ddot{Z}}$ and
\begin{equation}
\label{charactervaroperator}
X\perp \!\!\! \perp Y|Z ~\Leftrightarrow~ \Sigma_{XX| Z} - \Sigma_{XX| \ddot{Z}} = 0.    
\end{equation}
\end{lemma}
\noindent Following that, we go to obtain model selection criteria for regression in RKHS that can capture CI. We consider the following two 
\begin{equation}
\phi_{\mathcal{X}}(X) = g_1(Z)+U_1,~ \phi_{\mathcal{X}}(X) = g_2(\ddot{Z})+U_2,    
\end{equation}
where $U_1,U_2$ are the noise. Then with Eq.~(\ref{charactervaroperator}), we can derive that 
\begin{equation*}
X\perp \!\!\! \perp Y|Z ~\Leftrightarrow~ \mathbb{E}_{Z}[\text{Var}_{X|Z}[\phi_{\mathcal{X}}(X)|Z]] = \mathbb{E}_{\ddot{Z}}[\text{Var}_{X|\ddot{Z}}[\phi_{\mathcal{X}}(X)|\ddot{Z}]].   
\end{equation*}
Therefore, the CI relationship in the general case can be seen as a model selection problem for appropriate regression tasks. In the next section, we introduce the score functions for the model selection and show the computational challenges it faces.

\section{Computational Challenges of Score Functions}
\label{sec:generalizedscore}
We consider $n$ independent and identically distributed (i.i.d) samples $(\mathbf{x},\mathbf{z}) = \{(x_i,z_i)\}_{i=1}^{n}$ of random variables $(X,Z)$, then the regression in
RKHS on this finite observation is reformulated as
\begin{equation}
\label{eq_gen}
\mathbf{k_x} = h(\mathbf{z}) + U,
\end{equation}
where $\mathbf{k_x}:=[\phi_{\mathcal{X}}(x_1),\cdots, \phi_{\mathcal{X}}(x_n)]$ is the empirical feature map. Additionally, we define the $n\times n$ kernel matrices $\mathbf{K_X},\mathbf{K_Z}$ with entries $(k_{\mathcal{X}})_{ij}:= k_{\mathcal{X}}(x_i,x_j), (k_{\mathcal{Z}})_{ij}:= k_{\mathcal{Z}}(x_i,x_j)$ and the centered kernel matrices $\mathbf{\tilde{K}_X}:=\mathbf{H}\mathbf{K_X}\mathbf{H}, \mathbf{\tilde{K}_Z}:=\mathbf{H}\mathbf{K_Z}\mathbf{H}$, where
$\mathbf{H}=\mathbf{I}-\frac{1}{n}\mathbf{1}\mathbf{1}^T$ is the centering matrix and $\mathbf{1}$ is a vector of ones. Then, with kernel ridge regression, we can obtain the prediction as
\begin{equation}
\widehat{h}(z) = \mathbf{\tilde{K}_Z} \bigl(\mathbf{\tilde{K}_Z} + n\lambda \mathbf{I}\bigl)^{-1} \mathbf{k_x},
\end{equation}
where $\lambda$ is the regularization parameter. 


\noindent \textbf{Generalized score functions} are developed for the model selection to capture general conditional independence in RKHSs. The score can be derived as the log-likelihood function for the regression problem in Eq.~(\ref{eq_gen}), which capture the regression error of $U$. However, such a method may be not robust enough and could lead to over-fitting, as it uses the same data for both training the model and evaluation. To obtain a more reliable measure of model performance, two approaches are recommended: one uses cross-validated likelihood~\cite{huang2018generalized}, and the other focuses on maximizing the marginal likelihood~\cite{huang2018generalized, wangoptimal}. Both approaches are effective, though the marginal likelihood method requires an additional optimization process to find the optimal parameters. In this paper, we focus on cross-validated likelihood and analyze the computational challenges by considering it as an example. It is important to note that all generalized score functions face similar computational challenges.

\noindent \textbf{Cross-validated likelihood} can be described as follows. First, the entire dataset of $n$ samples is divided into a training set $(x^{1}, z^{1})$ and a testing set $(x^{0}, z^{0})$, with sample sizes of $n_1$ and $n_0$, respectively. The regression model is then trained using the training set and subsequently evaluated using the testing set. Formally, the evaluation result using the trained regression model is given by
\begin{equation}
\widehat{h}(z^0) = \mathbf{\tilde{K}_Z}^{0,1} \bigl(\mathbf{\tilde{K}_Z}^{1} + n\lambda \mathbf{I}\bigl)^{-1} \mathbf{k^1_x},   
\end{equation}
where $\mathbf{\tilde{K}_Z}^{1}:=\mathbf{k^1_z}(\mathbf{k^1_z})^T$ is the centered matrix of $x^{1}$, $(\mathbf{k^1_x},\mathbf{k^1_z})$ is the empirical feature map of $(x^{1},z^{1})$ and, $\mathbf{\tilde{K}_Z}^{0,1}:= \mathbf{k^0_z}(\mathbf{k^1_z})^T$. Note that subsequent similar symbols of the kernel matrix are defined by analogy. Also, the estimated covariance matrix of the residual is 
\begin{equation}
\begin{split}
 \hat{\mathbf{\Sigma}} &= \frac{1}{n_1} [\mathbf{k^1_x} - \widehat{h}(z^1)][\mathbf{k^1_x} - \widehat{h}(z^1)]^T \\&=  n_1\lambda^2 \bigl(\mathbf{\tilde{K}_Z}^{1} + n\lambda \mathbf{I}\bigl)^{-1}\mathbf{\tilde{K}_X}^{1} \bigl(\mathbf{\tilde{K}_Z}^{1} + n\lambda \mathbf{I}\bigl)^{-1}.
\end{split}
\end{equation}
For practical reasons, we add a small additional term to this estimate to ensure that it is positive-definite, i.e., using $\hat{\mathbf{\Sigma}}+\gamma \mathbf{I}$. Then,  the likelihood function can be obtained by 
\begin{equation}
\begin{split}
\label{losssssss}
l_{\text{CV}}(\widehat{h},\hat{\mathbf{\Sigma}}|x^0,z^0) = & -\frac{n_0^2}{2} \log(2\pi) - \frac{n_0}{2}\log \left|\hat{\mathbf{\Sigma}}+\gamma \mathbf{I} \right| \\
-&~\frac{1}{2\gamma}\text{Tr}\left\{[\mathbf{k^0_x} - \widehat{h}(z^1)](\hat{\mathbf{\Sigma}}+\gamma \mathbf{I})^{-1}[\mathbf{k^0_x} - \widehat{h}(z^1)]^T\right\}
\\  =  -\frac{n_0^2}{2} \log(2\pi) &- \frac{n_0}{2}\log \left|n_1\beta\mathbf{B}+\mathbf{I} \right| -\frac{n_0n_1}{2}\log \gamma  \\
- \frac{1}{2\gamma}\text{Tr}\Bigl\{\mathbf{\tilde{K}_X}^{0} &~+ \mathbf{\tilde{K}_Z}^{0,1}\mathbf{B}\mathbf{\tilde{K}_Z}^{1,0} -2\mathbf{\tilde{K}_X}^{0,1}\mathbf{A}\mathbf{\tilde{K}_Z}^{1,0}-n_1\beta\mathbf{\tilde{K}_X}^{0,1}\mathbf{C}\mathbf{\tilde{K}_X}^{1,0}\\-n_1\beta&\mathbf{\tilde{K}_Z}^{0,1}\mathbf{A}\mathbf{\tilde{K}_X}^{1}\mathbf{C}\mathbf{\tilde{K}_X}^{1}\mathbf{A}\mathbf{\tilde{K}_Z}^{1,0} + 2n_1\beta\mathbf{\tilde{K}_X}^{0,1}\mathbf{C}\mathbf{\tilde{K}_X}^{1}\mathbf{A}\mathbf{\tilde{K}_Z}^{1,0}\Bigl\},
\end{split}
\end{equation}
where $\mathbf{A} = (\mathbf{\tilde{K}_Z}^{1}+n_1\lambda\mathbf{I})^{-1}$, $\mathbf{B} =\mathbf{A}\mathbf{\tilde{K}_X}^{1}\mathbf{A}$,  $\mathbf{C} = \mathbf{A}(\mathbf{I}+n_1\beta \mathbf{B})^{-1}\mathbf{A}$ and $\beta:= \lambda^2/\gamma$. Also, when the conditional set $z$ is empty, we have 
\begin{equation}
\begin{split}
\label{loss2}
l_{\text{CV}}(\hat{\mathbf{\Sigma}}|x^0) =  -\frac{n_0^2}{2} \log(2\pi) &- \frac{n_0}{2}\log \left|\hat{\mathbf{\Sigma}}+\gamma \mathbf{I} \right| \\
-&~\frac{1}{2\gamma}\text{Tr}\left\{(\mathbf{k^0_x})(\hat{\mathbf{\Sigma}}+\gamma \mathbf{I})^{-1}(\mathbf{k^0_x})^T\right\}
\\  =  -\frac{n_0^2}{2} \log(2\pi) - \frac{n_0}{2}\log &\left|\frac{1}{n_1\gamma}\breve{\mathbf{B}}+\mathbf{I} \right| -\frac{n_0n_1}{2}\log \gamma  \\
- &\frac{1}{2\gamma}\text{Tr}\Bigl\{\mathbf{\tilde{K}_X}^{0} ~-\frac{1}{n_1\gamma}\mathbf{\tilde{K}_X}^{0,1}\breve{\mathbf{B}}\mathbf{\tilde{K}_X}^{1,0}\Bigl\},
\end{split}
\end{equation}
where $\breve{\mathbf{B}}:=\bigl(\mathbf{I}+\frac{1}{n_1\lambda}\mathbf{\tilde{K}_X}^{1} \bigl)^{-1}$. Typically, the above process is repeated $Q$ times and the average result is output as the final score, i.e., $S_{\text{CV}}(X,Z) = \frac{1}{Q}\sum_{q=1}^Q l_{\text{CV}}^{(q)}(\widehat{h},\hat{\mathbf{\Sigma}}|x^0,z^0)$. Now, we go to discuss the computational challenges posed by the score functions.

\noindent\textbf{Computational Challenges.} The space complexity for storing the kernel matrices is $\mathcal{O}(n^2)$. Additionally, the time complexity for calculating score function is analyzed as follows:
\begin{itemize}[leftmargin=5mm]
\item Calculating $\mathbf{\tilde{K}_X}$ and $\mathbf{\tilde{K}_Z}$ costs $\mathcal{O}(n^2)$. Note that $n=n_0+n_1$.  
\item Calculating $\mathbf{A}, \mathbf{B},\mathbf{C}\in \mathbb{R}^{n_1\times n_1}$ costs $\mathcal{O}(n_1^3)$ by the matrix inverse and matrix multiplication operations.
\item For the remaining items in $\text{Tr}\{\cdot \}$ the costs are as follows:
\begin{table}[h]
\small
\centering
\begin{tabular}{l|l}
\hline
 Terms & Complexity \\
\hline
$\mathbf{\tilde{K}_Z}^{0,1}\mathbf{B}\mathbf{\tilde{K}_Z}^{1,0}$, $\mathbf{\tilde{K}_X}^{0,1}\mathbf{A}\mathbf{\tilde{K}_Z}^{1,0}$, $\mathbf{\tilde{K}_X}^{0,1}\mathbf{C}\mathbf{\tilde{K}_X}^{1,0}$ & $\mathcal{O}(n_0n_1^2+n_0^2n_1)$   \\
$\mathbf{\tilde{K}_Z}^{0,1}\mathbf{A}\mathbf{\tilde{K}_X}^{1}\mathbf{C}\mathbf{\tilde{K}_X}^{1}\mathbf{A}\mathbf{\tilde{K}_Z}^{1,0}$ & $\mathcal{O}(n_0n_1^2+n_0^2n_1+n_1^3)$\\
$\mathbf{\tilde{K}_X}^{0,1}\mathbf{C}\mathbf{\tilde{K}_X}^{1}\mathbf{A}\mathbf{\tilde{K}_Z}^{1,0}$ & $\mathcal{O}(n_0n_1^2+n_0^2n_1+n_1^3)$\\
\hline
\end{tabular}
\label{tab:mascots}
\end{table}
\item In particular, for computing $\log \left|n_1\beta\mathbf{B}+\mathbf{I} \right|$, the Cholesky decomposition is utilized. This involves finding a lower triangular matrix $\mathbf{L}$ such that $n_1\beta\mathbf{B}+\mathbf{I} = \mathbf{L}\mathbf{L}^T$, which requires $\mathcal{O}(n_1^3)$ time. The logarithm of the determinant $\log \left|n_1\beta\mathbf{B}+\mathbf{I} \right|$ can then be computed as $\sum_{i=1}^{n_1}2\cdot \log L_{ii}$, which involves $\mathcal{O}(n_1)$ time cost. 
\item As a result, the overall time complexity is $\mathcal{O}(n^3)$. 
\end{itemize}
In summary, the primary computational challenges include the storage of kernel matrices, the operations required for matrix multiplication and inversion, and the calculation of determinants.

\section{Low-rank Approximated Kernel}
\label{sec:generalizedlowrank}
To address the computational bottleneck, we use low-rank approximation of the kernel function. Specifically, the Nyström method~\cite{williams2000using} is employed due to its efficiency in generating data-dependent approximations. Various Nyström method variants exist, broadly categorized into fixed sampling methods~\cite{kumar2012sampling} and adaptive sampling methods, and the latter is more flexible in selecting informative columns~\cite{drineas2005approximating,bach2002kernel}. Here, we choose the incomplete Cholesky decomposition (ICL) method~\cite{bach2002kernel}, detailed in Algorithm~\ref{alg1}.  
\begin{algorithm}
\caption{Kernel incomplete Cholesky decomposition (ICL)~\cite{bach2002kernel}}
\label{alg1}
\begin{flushleft}
\textbf{Input:} the kernel function $k$, the sample matrix $\mathbf{X}=\{\mathbf{x}_1,\mathbf{x}_2,..,\mathbf{x}_n\}$, the precision parameter $\eta$, the maximal rank parameter $m_0$\\
\textbf{Output:} an $n\times m$ matrix $\mathbf{\Lambda}$ such that $\Vert \mathbf{\Lambda}\mathbf{\Lambda}^T - \mathbf{K_{X}} \Vert\leq \eta$ if $m< m_0$.
\end{flushleft}
\begin{algorithmic}[1]
\State Initialization: permutation vector $\mathbf{\pi}=(1,...,n)$, diagonal of the residual kernel matrix $\mathbf{d}=\mathbf{0}_{n}$, $\mathbf{\Lambda}=\mathbf{O}_{n\times m_0}$
\For {$i=1,2,...,m_0$}
\State $\lhd$ \textbf{Update the diagonal of the residual kernel matrix.}
\State For $j\in \mathbf{\pi}[i: n]$, $\mathbf{d}_{j} = k(\mathbf{x}_j,\mathbf{x}_j)$.
\State If $i>1$:~for $j\in \mathbf{\pi}[i: n]$, $\mathbf{d}_{j} = \mathbf{d}_{j} - \sum_{r=1}^{i-1} \mathbf{\Lambda}_{jr}^2$.
\State If $\sum_{j=i}^n \mathbf{d}_j < \eta$: $m=i$, break. $\lhd$ \textbf{Already reach precision.}
\State $j^* = \arg \max_{j\in [i,n]} \mathbf{d}_{j}$. $\lhd$ \textbf{Find best new element.}
\State $\lhd$ \textbf{Permute elements $i$ and $j^*$.}
\State $\mathbf{\pi}[i] \leftrightarrow \mathbf{\pi}[j^*]$, and for $r\in [1:i-1]$,  $\mathbf{\Lambda}_{ir} \leftrightarrow \mathbf{\Lambda}_{j^*r}$. 
\State $\lhd$ \textbf{Calculate the $i$-th column.}
\State $\mathbf{\Lambda}_{ii} = \sqrt{\mathbf{d}_{j^*}}$, and $k^{\mathbf{\pi}}_{ji} = k(\mathbf{x}_{\mathbf{\pi}[j]}, \mathbf{x}_{\mathbf{\pi}[i]})$.
\State For $j\in [i+1:n]$, $\mathbf{\Lambda}_{ji} =(k^{\pi}_{ji} - \sum_{r=1}^{i-1}\mathbf{\Lambda}_{jr}\mathbf{\Lambda}_{ri})/\mathbf{\Lambda}_{ii}$.
\EndFor
\State If $m<m_0$:~$\mathbf{\Lambda} = \mathbf{\Lambda}[:,1:m]$; Else: $m=m_0$. $\lhd$ \textbf{Cut columns.}
\State $\mathbf{\Lambda} = \mathbf{\Lambda}[\mathbf{\pi}^{-1}[1:n],:]$ $\lhd$ \textbf{Reverse the permutation at end.}
\end{algorithmic}
\end{algorithm}
The algorithm works by sequentially selecting columns of $\mathbf{K_X}$, each is chosen to greedily maximize a lower bound on the reduction in approximation error. The running time of Algorithm~\ref{alg1} is 
 $\mathcal{O}(nm^2)$ , and the storage cost is $\mathcal{O}(nm)$ if we further using dynamic allocation for 
 $\mathbf{\Lambda}$. Applying for $\mathbf{K_{X}}$, we obtain $\mathbf{\Lambda_X}$ of size $n\times m_x$ such that $\mathbf{\Lambda}_X\mathbf{\Lambda}_X^T\thickapprox \mathbf{K_{X}}$. By centering $\mathbf{\Lambda}_X$ as $\tilde{\mathbf{\Lambda}}_{X} = \mathbf{H}\mathbf{\Lambda}_X = \mathbf{\Lambda}_X-\frac{1}{n}\mathbf{1}(\mathbf{1}^T\mathbf{\Lambda}_X)$, we achieve $\tilde{\mathbf{\Lambda}}_{X}\tilde{\mathbf{\Lambda}}_{X}^T \thickapprox \tilde{\mathbf{K}}_{X}$. Similarly, for $Z$, we obtain $\tilde{\mathbf{\Lambda}}_{Z}$ of size $n\times m_z$ such that $\tilde{\mathbf{\Lambda}}_{Z}\tilde{\mathbf{\Lambda}}_{Z}^T\thickapprox \tilde{\mathbf{K}}_{Z}$. This approach provides approximate kernels for score calculation and can be applied to any data type. Even though ICL has the advantage of generalizability as well as adaptivity, the speed of the algorithm is limited by the for loop (line 2, Alg.~\ref{alg1}), which does not allow it to receive the gain of matrix operations. Therefore, we design special algorithm for discrete variables to enable the application of the generalized score function to various data types more effectively. We first give the following results.
\begin{lemma}[Rank Bound for Discrete Case]
\label{lemma_rankbound_discrete}
Let $X$ be a discrete variable that takes on $m_d$ distinct values, and let the kernel function be $k$. The centered kernel matrix $\tilde{\mathbf{K}}_{X}$ satisfies $\text{rank}(\tilde{\mathbf{K}}_{X})\leq m_d$. 
\end{lemma}
\begin{proof}
Given that $X$ can take $m_d$ distinct values, the kernel matrix has at most $m_d$ linearly independence rows, thus $\text{rank}(\mathbf{K}_{X})\leq m_d$. Since $\tilde{\mathbf{K}}_{X} = \mathbf{H}{\mathbf{K}}_{X}\mathbf{H}$, we have $\text{rank}(\tilde{\mathbf{K}}_{X}) \leq \text{rank}(\mathbf{K}_{X})\leq m_d$.    
\end{proof}
Lemma~\ref{lemma_rankbound_discrete} indicates that for discrete variables, a low-rank decomposition of the kernel matrix can be obtained using samplings no more than the number of distinct values. To provide a more intuitive explanation, we offer an example as follows: 
\begin{example}
Let $X$ be be a binary discrete variable with samples $(x_1,x_2,x_3)=(1,0,1)$ and kernel function $k(x,y)=xy$. The decomposition of the kernel matrix is as follows:
\begin{equation}\label{eq:appendrow}
\newcommand\x{1}
\newcommand\y{0}
\mathbf{K}_{X} = \left[\begin{array}{ccc}
    \x  & \y  & \x  \\
    \y   & \y  & \y  \\
    \x   & \y   & \x \\
  \end{array}\right] = \left[\begin{array}{cc}
    \x  \\
    \y  \\
    \x   \\
  \end{array}\right]\left[\begin{array}{ccc}
    \x  & \y  & \x  \\

  \end{array}\right] = \mathbf{\Lambda}_X\mathbf{\Lambda}_X^T, 
\end{equation}
This provides a decomposition with rank no greater than $2$.
\end{example}
Now, we present an algorithm to obtain the desired low-rank decomposition for general kernel functions, as outlined in Algorithm~\ref{alg2}. First, we identify the important pivots by removing duplicate rows from $\mathbf{X}$, resulting in a submatrix $\mathbf{X}'$ with $m_d$ rows. Next, we compute $\mathbf{K^{-1/2}_{X'}}$, which can be achieved using Cholesky decomposition. The low-rank decomposition is then obtained such that $\mathbf{\Lambda} = \mathbf{K_{XX'}}\mathbf{L}^{-1}$. This process can be executed in $\mathcal{O}(nm^2+m^3)$ time with a storage cost of $\mathcal{O}(nm)$. Essentially, this method is similar to the Nyström method, with the distinction that the columns are directly obtained.
\begin{algorithm}
\caption{Low-rank decomposition for discrete variables}
\label{alg2}
\begin{flushleft}
\textbf{Input:} Kernel function $k$, sample matrix $\mathbf{X}=\{\mathbf{x}_1,\mathbf{x}_2,..,\mathbf{x}_n\}$ for discrete variable $X$ with $m_d$ distinct values.\\
\textbf{Output:} an $n\times m$ matrix $\mathbf{\Lambda}$ such that $\mathbf{\Lambda}\mathbf{\Lambda}^T= \mathbf{K_{X}}$.
\end{flushleft}
\begin{algorithmic}[1]
\State Remove duplicate rows from $\mathbf{X}$ to obtain $\mathbf{X}'$.
\State Let $m_d$ be the number of rows in $\mathbf{X}'$.
\State Compute $\mathbf{K_{XX'}} = k(\mathbf{X},\mathbf{X}')$ and $\mathbf{K_{X'}} = k(\mathbf{X}',\mathbf{X}')$.
\State Perform Cholesky decomposition of $\mathbf{K_{X'}} = \mathbf{L}\mathbf{L}^T$ .
\State Set $\mathbf{\Lambda} = \mathbf{K_{XX'}}\mathbf{L}^{-1}$ and $m=m_d$.
\end{algorithmic}
\end{algorithm}
Furthermore, we can demonstrate that the decomposition is accurate, as shown in the following lemma.
\begin{lemma}[Accurate Low-rank Decomposition]
For discrete variables, the decomposition is accurate, i.e., $\mathbf{K_{XX'}}\mathbf{K_{X'}^{-1}}\mathbf{K_{X'X}} = \mathbf{K_{X}}$.
\end{lemma}
\begin{proof}
Let the samples have values $\{v_1, v_2, ..., v_{m_d}\}$. Define $\mathbf{P}_{m_d\times n}$ with the elements $P_{ij} = \delta_{v_i=x_j}$, where $\delta$ is the {indicator} function. We verify that $\mathbf{K_{X'X}} = \mathbf{K_{X'}}\mathbf{P}$ since  
\begin{equation}
(\mathbf{K_{X'}}\mathbf{P})_{ij} = \sum_{q} k(v_i,v_q)\delta_{v_q=x_j} = k(v_i,x_j) = (\mathbf{K_{X'X}})_{ij}.   
\end{equation}
Thus, $\mathbf{K_{XX'}}\mathbf{K_{X'}^{-1}}\mathbf{K_{X'X}} = \mathbf{K_{XX'}}\mathbf{K_{X'}^{-1}}\mathbf{K_{X'}}\mathbf{P}= \mathbf{K_{XX'}}\mathbf{P}$. This completes the proof as
$(\mathbf{K_{XX'}}\mathbf{P})_{ij} = \sum_q k(x_i,v_q)\delta_{v_q=x_j} = k(x_i,x_j)$.
\end{proof}

Though we can achieve an accurate low-rank decomposition for discrete variables, as the dimensionality of the variable increases or the size of the conditional set grows, the number of discrete values can increase exponentially. Consequently, the low-rank decomposition for discrete cases is mainly suitable for speeding up computation on data with simpler structures. This method is particularly effective for sparse causal graphs, as demonstrated in the experimental section.

To address cases where the number of distinct values is very large, we employ Alg.~\ref{alg1} to obtain approximate results. Such a strategy is to balance accuracy and efficiency by trading off some precision for computational speed, thus providing a practical solution for handling more complex data structures efficiently.

\section{Score Function with Approximate Kernel}
\label{sec: scorefunctionswithapproximatekernel}
In this section, we show how to use the low-rank approximate kernel obtained above to speed up the calculation of generalized score function in Eq.~(\ref{losssssss}). We first present a theoretical roadmap, and then detail the derivation process. 
 
\noindent\textbf{Additional Notions.} We start by introducing some additional notations. Take the simplest item $\text{Tr}(\tilde{\mathbf{K}}_X^0)$ as an example. Using the low-rank algorithm introduced above, let $m_x$ be the number of pivots of kernel matrix of $X$ determined by the algorithm, we obtain a low-rank approximation of the centered kernel function $\tilde{\mathbf{K}}_X^0 \doteq \tilde{\mathbf{\Lambda}}_{X0}\tilde{\mathbf{\Lambda}}_{X0}^T$~\footnote{$\doteq$ indicates the equivalence after the substitution operation of all the centered kernel function into its corresponding approximate kernel function.},  where $\tilde{\mathbf{\Lambda}}_{X0}$ is an $n_0\times m_x$ matrix with rank no more than $m_x$, and $m_x\ll n_0$. Note that we use the subscript in $\tilde{\mathbf{\Lambda}}_{X0}$ to show the correspondence with the superscript in $\tilde{\mathbf{K}}_X^0$, thus similarly we have the approximation $\tilde{\mathbf{K}}_X^{0,1} \doteq \tilde{\mathbf{\Lambda}}_{X0}\tilde{\mathbf{\Lambda}}_{X1}^T$. Also, the notions for the approximation kernel of the training samples (related to $n_1$) and of the conditional variable $Z$ can be defined analogously. Additionally, we denote the number of pivots of kernel matrix of $Z$ determined by the low-rank algorithm as $m_z$.

\noindent\textbf{Theoretical roadmap.} 
For the simple case, the obtained approximate kernel function can be easily utilized. For example, for the calculating of $\text{Tr}(\tilde{\mathbf{\Lambda}}_{X0}\tilde{\mathbf{\Lambda}}_{X0}^T)$, since $\text{Tr}(\tilde{\mathbf{\Lambda}}_{X0}\tilde{\mathbf{\Lambda}}_{X0}^T) = \mathbf{1}^T(\tilde{\mathbf{\Lambda}}_{X0}\odot\tilde{\mathbf{\Lambda}}_{X0})\mathbf{1}$, where $\odot$ is the entrywise matrix product, we can calculate it with the time cost $\mathcal{O}(n_0m_x)$. As a contrast, if no approximation is performed, computing $\text{Tr}(\tilde{\mathbf{K}}_X^0)$ requires $\mathcal{O}(n_0^2)$ time, which is dominated by computing the centered kernel matrix $\tilde{\mathbf{K}}_X^0$.

For the remaining terms in trace operation in Eq.~(\ref{losssssss}), the problem is further complicated by a bulk of product operations and the intermediate terms $\mathbf{A} \doteq (\tilde{\mathbf{\Lambda}}_{Z1}\tilde{\mathbf{\Lambda}}_{Z1}^T+n_1\lambda\mathbf{I})^{-1}$, $\mathbf{B} \doteq \mathbf{A}\tilde{\mathbf{\Lambda}}_{X1}\tilde{\mathbf{\Lambda}}_{X1}^T\mathbf{A}$,  $\mathbf{C} \doteq \mathbf{A}(\mathbf{I}+n_1\beta \mathbf{B})^{-1}\mathbf{A}$ involving matrix inverses. To handle these complicated composite operations, we show that there is a way to convert them all into a specific form that can used for calculation efficiently and convenient to analyze. 

\begin{definition}[Dumbbell Form] 
We define the product of a sequence of matrices to have the dumbbell form when it can be expressed as $[\mathbf{\breve{W}_l}]_{n_l\times m_1}[\mathbf{\breve{T}_1}]_{m_1\times m_2}\cdots [\mathbf{\breve{T}_k}]_{m_k\times m_{k+1}}[\mathbf{\breve{W}_r}]_{m_{k+1}\times n_r}$, 
where the size $n_l, n_r\in \{n_0,n_1\}$ and $m_k\in \{m_x, m_z\}$ for all $k$, $k\geq 1$. For the case $k=0$, let the dumbbell form be $[\mathbf{\breve{W}_l}]_{n_l\times m_1}[\mathbf{\breve{W}_r}]_{m_{1}\times n_r}$.
\end{definition}
For simplicity, we call the product of a sequence of matrices of the dumbbell form as a dumbbell-form matrix chain. We have the following results for the dumbbell form. 
\begin{lemma}[Multiplicative Closure] 
\label{lem:multiply_operation}
The dumbbell form is closed for multiplication, i.e., the (well-defined) product of two dumbbell-form matrix chains remains a dumbbell-form matrix chain.
\end{lemma}
\begin{proof}
Just calculate the product of the leftmost matrix of one chain with the rightmost matrix of the other chain, then the proof can be done by the definition of the dumbbell form.   
\end{proof}
This result helps us to obtain the union dumbbell form efficiently from two dumbbell-form matrix chains. Calculating this union form requires to multiply $[\mathbf{\breve{W}_r}]_{m_{k+1}\times n_r}$ in the left chain by $[\mathbf{\breve{W}_l}]_{n_{l}\times m_1}$ in the right chain, which takes only $\mathcal{O}(nm^2)$ time. 

In addition, the dumbbell form can be related to the matrix inverse operations in the following way.
\begin{lemma}[inverse of dumbbell-form matrix chain] 
\label{lem:inverse_operation}
Let $\mathbf{\breve{W}_l}\mathbf{\breve{T}_1}\cdots\mathbf{\breve{T}_k}\mathbf{\breve{W}_r}$ be a dumbbell-form matrix chain, and assume that all the matrices are conformable, thus ensuring the following inverse operations are well-defined. Then, the matrix inverse operation with regularization $(\mathbf{I} + \mathbf{\breve{W}_l}\mathbf{\breve{T}_1}\cdots\mathbf{\breve{T}_k}\mathbf{\breve{W}_r})^{-1}$ can be expressed as $\mathbf{I}$ plus a dumbbell-form matrix chain. 
\end{lemma}
\begin{proof}
Here, the Woodbury matrix identity is used, which is 
\begin{equation}
\label{woodburymatrixidentity}
    (\mathbf{I} + \mathbf{\breve{U}}\mathbf{\breve{V}})^{-1} = \mathbf{I} - \mathbf{\breve{U}} (\mathbf{I} + \mathbf{\breve{V}}\mathbf{\breve{U}} )^{-1} \mathbf{\breve{V}}
\end{equation}
where $\mathbf{\breve{U}}, \mathbf{\breve{V}}$ are conformable matrices. For any partition, i.e., let $\mathbf{\breve{U}} = \mathbf{\breve{W}_l}\mathbf{\breve{T}_1}\cdots \mathbf{\breve{T}_i}$, $i\leq k$, and denote $\mathbf{\breve{V}}$ as the remain part, we can check that the size $\mathbf{\breve{U}}$ is $n_l\times m_{i+1}$, the size $\mathbf{\breve{V}}$ is $m_{i+1}\times n_r$ and thus the size of $(\mathbf{I} + \mathbf{\breve{V}}\mathbf{\breve{U}} )^{-1}$ is $m_{i+1}\times m_{i+1}$, which completes the proof.
\end{proof}

Combining Lemma~\ref{lem:multiply_operation} and~\ref{lem:inverse_operation}, we now show how to transform all terms in Eq.~(\ref{losssssss}) to a sum of dumbbell-form matrix chains.  For easy understanding, we provide an example as follows.  
\begin{example} Take $\mathbf{\tilde{K}_X}^{0,1}\mathbf{A}\mathbf{\tilde{K}_Z}^{1,0}\doteq \tilde{\mathbf{\Lambda}}_{X0}\tilde{\mathbf{\Lambda}}_{X1}^T\mathbf{A} \tilde{\mathbf{\Lambda}}_{Z1}\tilde{\mathbf{\Lambda}}_{Z0}^T$ for example. Using the Woodbury matrix identity as in Eq.~(\ref{woodburymatrixidentity}), we have
\begin{equation}
\label{a_dumb}
\mathbf{A} \doteq \frac{1}{n_1\lambda} \left[\mathbf{I} - \tilde{\mathbf{\Lambda}}_{Z1}\bigl( n_1\lambda\mathbf{I} +\tilde{\mathbf{\Lambda}}_{Z1}^T \tilde{\mathbf{\Lambda}}_{Z1} \bigl)^{-1}\tilde{\mathbf{\Lambda}}_{Z1}^T\right].
\end{equation}
By Lemma~\ref{lem:multiply_operation}, $\tilde{\mathbf{\Lambda}}_{X0}[\tilde{\mathbf{\Lambda}}_{X1}^T\tilde{\mathbf{\Lambda}}_{Z1}]\tilde{\mathbf{\Lambda}}_{Z0}^T$ and $\tilde{\mathbf{\Lambda}}_{X0}[\tilde{\mathbf{\Lambda}}_{X1}^T\tilde{\mathbf{\Lambda}}_{Z1}]\bigl( n_1\lambda\mathbf{I} +\tilde{\mathbf{\Lambda}}_{Z1}^T \tilde{\mathbf{\Lambda}}_{Z1} \bigl)^{-1}$ $[\tilde{\mathbf{\Lambda}}_{Z1}^T\tilde{\mathbf{\Lambda}}_{Z1}]\tilde{\mathbf{\Lambda}}_{Z0}^T$ are both dumbbell form, thus $\mathbf{\tilde{K}_X}^{0,1}\mathbf{A}\mathbf{\tilde{K}_Z}^{1,0}$ can be expressed as a sum of two dumbbell-form matrix chains. 
\end{example}
For the term $\mathbf{B}$, using Eq.~(\ref{a_dumb}), we can obtain $[\mathbf{A}\tilde{\mathbf{\Lambda}}_{X1}]_{n_1\times m_x}$ in $\mathcal{O}(n_1m_x^2)$ time, and then $\mathbf{B}\doteq [\mathbf{A}\tilde{\mathbf{\Lambda}}_{X1}][\mathbf{A}\tilde{\mathbf{\Lambda}}_{X1}]^T$ is the dumbbell form. Using Lemma~\ref{lem:inverse_operation} and~\ref{lem:multiply_operation}, we can show that $\mathbf{C}$ can be expressed as $\mathbf{I}$ plus a sum of dumbbell-form matrix chains.  Then, similarly to the example above, we can show that all the terms in Eq.~(\ref{losssssss}) are also a sum of dumbbell-form matrix chains. 

\noindent \textbf{Computational Complexity.} Next, we explore how to perform computation effectively with the dumbbell form. We denote $m=\max\{m_x,m_z\}$. Note that according to the idea introduced before, to obtain the dumbbell form, we need $\mathcal{O}(nm^2+m^3)$ time that $\mathcal{O}(nm^2)$ is used for obtaining the union dumbbell form using Lemma~\ref{lem:multiply_operation}, and $\mathcal{O}(m^3)$ is used to calculate matrix inverse within a dumbbell-form matrix chain using Lemma~\ref{lem:inverse_operation}. For trace calculation, the cyclic property of trace operation can be used,  which is
\begin{equation}
\text{Tr}(\mathbf{\breve{W}_l}\mathbf{\breve{T}_1}\cdots\mathbf{\breve{T}_k}\mathbf{\breve{W}_r}) = \text{Tr}(\mathbf{\breve{W}_r}\mathbf{\breve{W}_l}\mathbf{\breve{T}_1}\cdots\mathbf{\breve{T}_k}) 
\end{equation}
Note that $[\mathbf{\breve{W}_r}\mathbf{\breve{W}_l}]$ has the size $m_{k+1}\times m_1$ and can be calculated in $\mathcal{O}(nm^2)$. Therefore, the time complexity to calculate the object in trace is $\mathcal{O}(nm^2)$ by matrix multiplication. Also, for calculating determinant, the Weinstein–Aronszajn identity is used, which is 
\begin{equation}
\label{det_eq}
|\mathbf{I}+\mathbf{\breve{U}}\mathbf{\breve{V}}|=|\mathbf{I}+\mathbf{\breve{V}}\mathbf{\breve{U}}|,
\end{equation}
where $\mathbf{\breve{U}} $ and $\mathbf{\breve{V}}$ are conformable matrices. Similar to the way we deal with Eq.~(\ref{woodburymatrixidentity}), we partition the dumbbell-form matrix chain to obtain $\mathbf{\breve{U}},\mathbf{\breve{V}}$, then the result $\mathbf{I}+\mathbf{\breve{V}}\mathbf{\breve{U}}$ is of size $m_{i+1}\times m_{i+1}$. Subsequent calculations can be done in $\mathcal{O}(m^3)$ time with the Cholesky decomposition. As a result, we can calculate Eq.~(\ref{losssssss}) with the time cost $\mathcal{O}(nm^2)$. Additionally, the storage cost is $\mathcal{O}(nm)$,  which is dominated by the storage of the approximate kernels.

\noindent\textbf{Detailed Computation.} Above, we provide a theoretical route, next we perform the detailed computation. 

\noindent \textbf{\textit{Results when $|z|\neq 0$.}} As above, we first obtain the dumbbell form of $\mathbf{A},\mathbf{B},\mathbf{C}$ in Eq.~(\ref{losssssss}). We define the following notions that are the key components of a dumbbell-form matrix chain. 
\begin{table}[h]
\footnotesize
 \renewcommand\arraystretch{1.5}
\centering
\scalebox{0.80}{
\begin{tabular}{l|cccccc}
\hline
Symbol & $\mathbf{P}$ & $\mathbf{E}$ & $\mathbf{F}$ &$\mathbf{V}$ & $\mathbf{U}$ & $\mathbf{S}$ \\
\hline
 Term & $\tilde{\mathbf{\Lambda}}_{X1}^T\tilde{\mathbf{\Lambda}}_{X1}$ & $\tilde{\mathbf{\Lambda}}_{Z1}^T\tilde{\mathbf{\Lambda}}_{X1}$ & $\tilde{\mathbf{\Lambda}}_{Z1}^T\tilde{\mathbf{\Lambda}}_{Z1}$ & $\tilde{\mathbf{\Lambda}}_{X0}^T\tilde{\mathbf{\Lambda}}_{X0}$ & $\tilde{\mathbf{\Lambda}}_{Z0}^T\tilde{\mathbf{\Lambda}}_{X0}$ & $\tilde{\mathbf{\Lambda}}_{Z0}^T\tilde{\mathbf{\Lambda}}_{Z0}$ \\
  Size & $m_x\times m_x$ & $m_z\times m_x$ & $m_z\times m_z$ & $m_x\times m_x$ & $m_z\times m_x$& $m_z\times m_z$ \\
 Time & $\mathcal{O}(n_1m_x^2)$ & $\mathcal{O}(n_1m_xm_z)$ & $\mathcal{O}(n_1m_z^2)$ & $\mathcal{O}(n_0m_x^2)$ & $\mathcal{O}(n_0m_xm_z)$ & $\mathcal{O}(n_0m_z^2)$ \\
\hline
\end{tabular}}
\label{tab:sdas}
\end{table}

Then, we have $({n_1\lambda})\mathbf{A}\doteq  \mathbf{I}- \tilde{\mathbf{\Lambda}}_{Z1}\mathbf{D}\tilde{\mathbf{\Lambda}}_{Z1}^T$ as in Eq.~(\ref{a_dumb}), where we denote $\mathbf{D}:=\bigl(n_1\lambda\mathbf{I} +\mathbf{F}\bigl)^{-1}$. Also, ${(n_1\lambda)^2}\mathbf{B}$ of the dumbbell form is
\begin{equation*}
\begin{split}
{(n_1\lambda)^2}\mathbf{B} \doteq [\mathbf{A}\tilde{\mathbf{\Lambda}}_{X1}][\tilde{\mathbf{\Lambda}}_{X1}^T\mathbf{A}] =  (\tilde{\mathbf{\Lambda}}_{X1} -\tilde{\mathbf{\Lambda}}_{Z1}\mathbf{D}\mathbf{E})(\tilde{\mathbf{\Lambda}}_{X1} -\tilde{\mathbf{\Lambda}}_{Z1}\mathbf{D}\mathbf{E})^T.
\end{split}
\end{equation*}
Next, we calculate $\mathbf{C}\doteq \mathbf{A}(\mathbf{I}+n_1\beta \mathbf{B})^{-1}\mathbf{A}$. We can handle the inverse operation according to the Lemma~\ref{lem:inverse_operation}, such that
\begin{equation}
\begin{split}
(\mathbf{I}+n_1\beta &\mathbf{B})^{-1} \doteq  (\mathbf{I}+n_1\beta [\mathbf{A}\tilde{\mathbf{\Lambda}}_{X1}][\tilde{\mathbf{\Lambda}}_{X1}^T\mathbf{A}])^{-1}  \\&= \mathbf{I}-n_1\beta [\mathbf{A}\tilde{\mathbf{\Lambda}}_{X1}](\mathbf{I}+n_1\beta\tilde{\mathbf{\Lambda}}_{X1}^T\mathbf{A}^2\tilde{\mathbf{\Lambda}}_{X1})^{-1}[\tilde{\mathbf{\Lambda}}_{X1}^T\mathbf{A}].
\end{split}
\end{equation}
For simplification, we denote $\mathbf{G}:=\bigl(\mathbf{I}+n_1\beta\tilde{\mathbf{\Lambda}}_{X1}^T\mathbf{A}^2\tilde{\mathbf{\Lambda}}_{X1}\bigl)^{-1}$, and the term for calculating $\mathbf{G}$ can be obtained by 
\begin{equation}
\label{term_G}
(n_1\lambda)^2\cdot \tilde{\mathbf{\Lambda}}_{X1}^T\mathbf{A}^2\tilde{\mathbf{\Lambda}}_{X1} = \mathbf{P} - 2\mathbf{E}^T\mathbf{D}\mathbf{E} + \mathbf{E}^T\mathbf{D}\mathbf{F}\mathbf{D}\mathbf{E}.    
\end{equation}
\noindent As a result, using Lemma~\ref{lem:multiply_operation}, the term $\mathbf{C}$ can be expressed as $\mathbf{I}$ plus a sum of dumbbell-form matrix chains, i.e. 
\begin{equation}
\mathbf{C} \doteq \frac{1}{(n_1\lambda)^2}\mathbf{I} + \tilde{\mathbf{\Lambda}}_{Z1}\mathfrak{A}\tilde{\mathbf{\Lambda}}_{Z1}^T +   \tilde{\mathbf{\Lambda}}_{X1}\mathfrak{D}\tilde{\mathbf{\Lambda}}_{X1}^T + \tilde{\mathbf{\Lambda}}_{X1}\mathfrak{B}\tilde{\mathbf{\Lambda}}_{Z1}^T + \tilde{\mathbf{\Lambda}}_{Z1}\mathfrak{C}\tilde{\mathbf{\Lambda}}_{X1}^T,
\end{equation}
where the terms is defined as 
\begin{equation}
\begin{split}
&\mathfrak{A} :=  -\frac{1}{(n_1\lambda)^2}(2\mathbf{I}-\mathbf{D} \mathbf{F})\mathbf{D} - \frac{n_1\beta}{ (n_1\lambda)^4} (2\mathbf{I}-\mathbf{D} \mathbf{F})\mathbf{D}\mathbf{E}\mathbf{G}\mathbf{E}^T\mathbf{D}(2\mathbf{I}-\mathbf{D} \mathbf{F})^T,\\    
&\mathfrak{B} := \frac{n_1\beta}{ (n_1\lambda)^4}\cdot\mathbf{G}\mathbf{E}^T\mathbf{D}(2\mathbf{I}-\mathbf{D}\mathbf{F})^T,~ \mathfrak{C} :=\mathfrak{B}^T,~\mathfrak{D}:=-\frac{n_1\beta}{ (n_1\lambda)^4}\mathbf{G}.
\end{split}
\end{equation}
After obtaining $\mathbf{A}, \mathbf{B}, \mathbf{C}$, we can obtain all the terms in Eq.~(\ref{losssssss}). The remain thing is to compute determinants as well as traces. 

For the term $\log \left|n_1\beta\mathbf{B}+\mathbf{I} \right|$, according to the Weinstein–Aronszajn identity as in Eq.~(\ref{det_eq}), we have 
\begin{equation}
\left|\mathbf{I}+n_1\beta\mathbf{B}\right| = |\mathbf{I}+n_1\beta \mathbf{A}\tilde{\mathbf{\Lambda}}_{X1}\tilde{\mathbf{\Lambda}}_{X1}^T\mathbf{A}| = |\mathbf{I}+n_1\beta \tilde{\mathbf{\Lambda}}_{X1}^T\mathbf{A}^2\tilde{\mathbf{\Lambda}}_{X1}|.
\end{equation}
By combining with Eq.~(\ref{term_G}), we have $\mathbf{Q} := \mathbf{I}+ n_1\beta \tilde{\mathbf{\Lambda}}_{X1}^T\mathbf{A}^2\tilde{\mathbf{\Lambda}}_{X1}$ as 
\begin{equation}
\mathbf{Q}  = \mathbf{I} + \frac{1}{n_1\gamma}\left(\mathbf{P} - 2\cdot \mathbf{E}^T\mathbf{D}\mathbf{E} + \mathbf{E}^T\mathbf{D}\mathbf{F}\mathbf{D}\mathbf{E}\right).  
\end{equation}
Then, using Cholesky decomposition, i.e. $\mathbf{Q} = \mathbf{R}\mathbf{R}^T$, we obtain the results $\log \left|n_1\beta\mathbf{B}+\mathbf{I} \right| = \log |\mathbf{Q}|=  2\cdot \log |\mathbf{R}|  = \sum_{i=1}^m 2\cdot\log \mathbf{R}_{ii}.$
And for the terms related to the trace operation, we list the results of each term as follows. For simplification, we denote $\mathbf{W}:= \frac{1}{(n_1\lambda)^2}\mathbf{P}+\mathbf{E}^T\mathfrak{A}\mathbf{E} + \mathbf{P}\mathfrak{D}\mathbf{P} + \mathbf{P}\mathfrak{B}\mathbf{E} + \mathbf{E}^T\mathfrak{C}\mathbf{P}$. Concretely,  
\begin{itemize}[leftmargin=5mm]
\item For the terms $\tilde{\mathbf{K}}_X^0, \mathbf{\tilde{K}_Z}^{0,1}\mathbf{B}\mathbf{\tilde{K}_Z}^{1,0}, \mathbf{\tilde{K}_X}^{0,1}\mathbf{A}\mathbf{\tilde{K}_Z}^{1,0}$, we have 
\begin{equation}
\begin{split}
&\text{Tr}{(\tilde{\mathbf{K}}_X^0)} = \text{Tr}{(\tilde{\mathbf{\Lambda}}_{X0}^{T}\tilde{\mathbf{\Lambda}}_{X0})} = \text{Tr}(\mathbf{V})\\
&(n_1\lambda)\cdot \text{Tr}{(\mathbf{\tilde{K}_X}^{0,1}\mathbf{A}\mathbf{\tilde{K}_Z}^{1,0})} =  \text{Tr}{\Bigl(\mathbf{U}\mathbf{E}^T(\mathbf{I}-\mathbf{D}\mathbf{F})\Bigl)}\\
&(n_1\lambda)^2\cdot \text{Tr}{(\mathbf{\tilde{K}_Z}^{0,1}\mathbf{B}\mathbf{\tilde{K}_Z}^{1,0})}=  \text{Tr}\Bigl(\mathbf{S}(\mathbf{I}-\mathbf{D}\mathbf{F})^T\mathbf{E}\mathbf{E}^T(\mathbf{I}-\mathbf{D}\mathbf{F})\Bigl).  
\end{split}
\end{equation}
\item For the term $\mathbf{\tilde{K}_X}^{0,1}\mathbf{C}\mathbf{\tilde{K}_X}^{1,0}$, we have
\begin{equation}
\text{Tr}\left(\mathbf{\tilde{K}_X}^{0,1}\mathbf{C}\mathbf{\tilde{K}_X}^{1,0}\right) = \text{Tr}\left(\mathbf{V}\mathbf{W}\right).    
\end{equation}
\item For the term $\mathbf{\tilde{K}_X}^{0,1}\mathbf{C}\mathbf{\tilde{K}_X}^{1}\mathbf{A}\mathbf{\tilde{K}_Z}^{1,0}$, we have
\begin{equation}
\begin{split}
(n_1&\lambda)\cdot\text{Tr}(\mathbf{\tilde{K}_X}^{0,1}\mathbf{C}\mathbf{\tilde{K}_X}^{1}\mathbf{A}\mathbf{\tilde{K}_Z}^{1,0}) =\text{Tr}{\left(\mathbf{U}\mathbf{W}\mathbf{E}^T(\mathbf{I}-\mathbf{D}\mathbf{F})\right). }   
\end{split}
\end{equation}
\item For the term $\mathbf{\tilde{K}_Z}^{0,1}\mathbf{A}\mathbf{\tilde{K}_X}^{1}\mathbf{C}\mathbf{\tilde{K}_X}^{1}\mathbf{A}\mathbf{\tilde{K}_Z}^{1,0}$, we have
\begin{equation}
\begin{split}
(n_1\lambda)^2\cdot\text{Tr}(\mathbf{\tilde{K}_Z}^{0,1}&\mathbf{A}\mathbf{\tilde{K}_X}^{1}\mathbf{C}\mathbf{\tilde{K}_X}^{1}\mathbf{A}\mathbf{\tilde{K}_Z}^{1,0}) \\ = ~&\text{Tr}\left(\mathbf{S}(\mathbf{I} - \mathbf{D}\mathbf{F})^T\mathbf{E}\mathbf{W}\mathbf{E}^T(\mathbf{I} - \mathbf{D}\mathbf{F})\right). 
\end{split}
\end{equation}
\end{itemize}
As a result, the overall results of trace operation is given by
\begin{equation}
\begin{split}
\text{Tr}~\Bigl\{\mathbf{\tilde{K}_X}^{0} + \mathbf{\tilde{K}_Z}^{0,1}\mathbf{B}\mathbf{\tilde{K}_Z}^{1,0} -&~2\mathbf{\tilde{K}_X}^{0,1}\mathbf{A}\mathbf{\tilde{K}_Z}^{1,0}-n_1\beta\mathbf{\tilde{K}_X}^{0,1}\mathbf{C}\mathbf{\tilde{K}_X}^{1,0}\\-n_1\beta\mathbf{\tilde{K}_Z}^{0,1}\mathbf{A}\mathbf{\tilde{K}_X}^{1}&\mathbf{C}\mathbf{\tilde{K}_X}^{1}\mathbf{A}\mathbf{\tilde{K}_Z}^{1,0} + 2n_1\beta\mathbf{\tilde{K}_X}^{0,1}\mathbf{C}\mathbf{\tilde{K}_X}^{1}\mathbf{A}\mathbf{\tilde{K}_Z}^{1,0}\Bigl\}
\\= \text{Tr}\Bigl[(\mathbf{I}-n_1\beta\mathbf{W})\Bigl(\mathbf{V}&-\frac{2}{n_1\lambda}\mathbf{E}^T(\mathbf{I}-\mathbf{D}\mathbf{F})\mathbf{U}\\&+\frac{1}{(n_1\lambda)^2}\mathbf{E}^T(\mathbf{I}-\mathbf{D}\mathbf{F})\mathbf{S}(\mathbf{I}-\mathbf{D}\mathbf{F})^T\mathbf{E}\Bigl)\Bigl].
\end{split}    
\end{equation}
Until now, we obtain the generalized score function for the non-empty condition set case with approximate kernel that can be calculated $\mathcal{O}(nm^2)$ time with $\mathcal{O}(nm)$ storage. 

\noindent \textbf{\textit{Results when $|z|=0$.}} Next, we obtain the results for Eq.~(\ref{loss2}) that the conditional set is empty. According to Lemma.~\ref{lem:inverse_operation}, using the Woddbury matrix identity as in Eq.~(\ref{woodburymatrixidentity}), we have
\begin{equation}
\breve{\mathbf{B}}\doteq  \left(\mathbf{I}+\frac{1}{n_1\lambda}\tilde{\mathbf{\Lambda}}_{X1}\tilde{\mathbf{\Lambda}}_{X1}^T\right)^{-1} = \mathbf{I} - \frac{1}{n_1\lambda} \tilde{\mathbf{\Lambda}}_{X1}\breve{\mathbf{D}}\tilde{\mathbf{\Lambda}}_{X1}^T,   
\end{equation}
where $\breve{\mathbf{D}}:=\bigl(\mathbf{I}+\frac{1}{n_1\lambda}\tilde{\mathbf{\Lambda}}_{X1}^T\tilde{\mathbf{\Lambda}}_{X1} \bigl)^{-1}$. Then for the term $\left|\mathbf{I}+\frac{1}{n_1\lambda}\mathbf{\tilde{K}_X}^{1}\right|$, using the Weinstein–Aronszajn identity as in Eq.~(\ref{det_eq}), we have
\begin{equation}
\left|\mathbf{I}+\frac{1}{n_1\lambda}\mathbf{\tilde{K}_X}^{1}\right| \doteq \left|\mathbf{I}+\frac{1}{n_1\lambda}\tilde{\mathbf{\Lambda}}_{X1}\tilde{\mathbf{\Lambda}}_{X1}^T\right| =  \left|\mathbf{I}+\frac{1}{n_1\lambda}\tilde{\mathbf{\Lambda}}_{X1}^T\tilde{\mathbf{\Lambda}}_{X1}\right|,  
\end{equation}
where the size of $\tilde{\mathbf{\Lambda}}_{X1}^T\tilde{\mathbf{\Lambda}}_{X1}$ is $m_x\times m_x$. Therefore, calculate the $\breve{\mathbf{Q}} := \mathbf{I}+\frac{1}{n_1\lambda}\tilde{\mathbf{\Lambda}}_{X1}^T\tilde{\mathbf{\Lambda}}_{X1}$ cost $\mathcal{O}(nm^2)$.
Then, using Cholesky decomposition, i.e. $\breve{\mathbf{Q}} = \breve{\mathbf{R}}\breve{\mathbf{R}}^T$, we obtain the results $\log \left|\mathbf{I}+\frac{1}{n_1\lambda}\mathbf{\tilde{K}_X}^{1} \right| \doteq \log |\breve{\mathbf{Q}}|=  2\cdot \log |\breve{\mathbf{R}}|  = \sum_{i=1}^m 2\cdot\log \breve{\mathbf{R}}_{ii}$. This process cost $\mathcal{O}(m^3)$ time. As a result, the overall time cost to calculate $\log \left|\mathbf{I}+\frac{1}{n_1\lambda}\tilde{\mathbf{\Lambda}}_{X1}\tilde{\mathbf{\Lambda}}_{X1}^T \right|$ is $\mathcal{O}(nm^2)$. 
And for the terms related to the trace operation, we list the results of each term as follows. Concretely,  
\begin{itemize}[leftmargin=5mm]
\item For the term $\tilde{\mathbf{K}}_X^0$, using the cycle property of trace, we have 
\begin{equation}
\begin{split}
&\text{Tr}{(\tilde{\mathbf{K}}_X^0)} = \text{Tr}{(\tilde{\mathbf{\Lambda}}_{X0}^{T}\tilde{\mathbf{\Lambda}}_{X0})} = \text{Tr}(\mathbf{V}), 
\end{split}
\end{equation}
where we define $\mathbf{V}: = \tilde{\mathbf{\Lambda}}_{X0}^{T}\tilde{\mathbf{\Lambda}}_{X0}$ with size of $m_x\times m_x$.  
\item For the term $\mathbf{\tilde{K}_X}^{0,1}\breve{\mathbf{B}}\mathbf{\tilde{K}_X}^{1,0}$, we have 
\begin{equation}
\begin{split}
\text{Tr}\Bigl(\mathbf{\tilde{K}_X}^{0,1}\mathbf{B}\mathbf{\tilde{K}_X}^{1,0}\Bigl) &=  \text{Tr}\left(\tilde{\mathbf{\Lambda}}_{X0}\tilde{\mathbf{\Lambda}}_{X1}^{T}\left(\mathbf{I} - \frac{1}{n_1\lambda} \tilde{\mathbf{\Lambda}}_{X1}\mathbf{D}\tilde{\mathbf{\Lambda}}_{X1}^T\right)\tilde{\mathbf{\Lambda}}_{X1}\tilde{\mathbf{\Lambda}}_{X0}^{T}\right)\\&=  \text{Tr}\left(\mathbf{V}\mathbf{P} - \frac{1}{n_1\lambda}\mathbf{V}\mathbf{P}\breve{\mathbf{D}}\mathbf{P} \right),         
\end{split}
\end{equation}
where $\mathbf{P}:= \tilde{\mathbf{\Lambda}}_{X1}^T\tilde{\mathbf{\Lambda}}_{X1}$ with size of $m_x\times m_x$. 
\end{itemize}
We can check that the time cost within the trace operation is $\mathcal{O}(m^3)$ and obtain $\mathbf{V}, \mathbf{P}$ cost $\mathcal{O}(nm^2)$ time. As a  result, we provide the calculation process of our score for the case that the conditional set is empty and show that the overall time cost is $\mathcal{O}(nm^2)$. 

\section{Causal Structure Search}
\label{sec:causalsearch}
Above, we obtain the generalized score function with approximate kernel. Next, we introduce how to apply it to causal structure learning. We assume that there is no feedback or unobserved common cause in the underlying causal graph.  

Suppose that we have a $\text{DAG}$ $\mathcal{G}$ with $d$ variables, $X_1,X_2,...,X_d$, and a dataset $\mathcal{D}$ consisting of $n$ i.i.d. observations of these variables. Let the current hypothetical $\text{DAG}$ be $\mathcal{G}_h$. We obtain the local score of variable $X_i$ and its parents $\text{Pa}_i$ by setting $(X, Z)$ in the previously obtained low-rank kernel-based score function to $(X_i, \text{Pa}_i)$. We denote this local score as ${S}_{\text{LR}}(X_i,\text{Pa}_i)$. Note that we allow $\text{Pa}_i$ to be an empty set, this corresponds to the case that $Z$ is empty. Then, following the decomposable property of the score function, i.e., it can be written as a sum of measures, which is a function of only one variable and its parents, we can obtain the score function of $\text{DAG}$ $\mathcal{G}_h$ as 
\begin{equation}
\label{score_lowrank}
{S}_{\text{LR}}(\mathcal{G}_h, \mathcal{D}) = \sum_{i=1}^d {S}_{\text{LR}}(X_i,\text{Pa}^{\mathcal{G}_h}_i),    
\end{equation} 
where $\text{Pa}^{\mathcal{G}_h}_i$ indicates the parents of $X_i$ in the graph $\mathcal{G}_h$. Given this score function, we are concerned about whether it enables the search procedure to find the correct equivalence class consistent with data distribution as $n \rightarrow \infty$. To see that, we first introduce an important property for a score function named local consistency~\cite{chickering2002optimal}, which represents whether a score function is a suitable optimization criterion that faithfully reflects structural correctness.
Then, we will show that when ${S}_{\text{LR}}(\mathcal{G}_h, \mathcal{D})$ is a good enough approximation of the CV likelihood score ${S}_{\text{CV}}(\mathcal{G}_h, \mathcal{D})$, ${S}_{\text{LR}}(\mathcal{G}_h, \mathcal{D})$ has a locally consistent property. 

\noindent \textbf{Score Local Consistency.} The local consistency property ensures that for two DAGs only different in one edge, the DAG that gives one more correct (conditional) independence constraint has a larger score. Formally, the score local consistency is defined as follows. 
\begin{definition}[Score Local Consistency~\cite{chickering2002optimal}] Let $\mathcal{G}$ be any DAG, and $\mathcal{G}'$ be the DAG that results from adding the edge $X_i \rightarrow X_j$
on $\mathcal{G}$,  $\mathcal{D}$ be the dataset from distribution $p(\cdot)$. A score function $S(\mathcal{G};\mathcal{D})$ is locally consistent if the following two properties hold as the sample size $n\rightarrow \infty$:
\begin{enumerate}
    \item If $X_j \not\!\perp\!\!\!\perp X_i|PA^\mathcal{G}_j$, then $S(\mathcal{G}';\mathcal{D}) > S (\mathcal{G};\mathcal{D})$.
    \item If $X_j \perp \!\!\! \perp X_i|PA^\mathcal{G}_j$, then $S(\mathcal{G};\mathcal{D}) > S (\mathcal{G}';\mathcal{D})$.
\end{enumerate}
\end{definition}
It has been shown that under {all conditions given in Lemma 3 in~\cite{huang2018generalized}}, the CV likelihood score is locally consistent. With a good enough approximation, ${S}_{\text{LR}}(\mathcal{G}_h, \mathcal{D})$ also naturally satisfies the local consistency. 
Furthermore, when there are significant local differences in scores, the necessary approximation accuracy may be reduced by preserving the relative magnitudes of the scores. 

When the score is locally consistent, the greedy equivalence search (GES) algorithm~\cite{chickering2002optimal,chickering2020statistically} guarantees to find the Markov equivalence class that is consistent with the data generative distribution asymptotically~\cite{huang2018generalized}. The reason is that GES searches for
the maximum local score change between adjacent equivalence classes.

\noindent \textbf{GES algorithm overview:} The algorithm starts with an empty graph and initializes the score. It then performs two phases: forward and backward. In the forward phase, edges are incrementally added, with each addition evaluated for the best improvement in score. This continues until no further improvements are possible. In the backward phase, edges are iteratively removed, with each removal evaluated for the best score improvement, continuing until no further improvements can be made.

As a result, the GES search procedure combined with our score ${S}_{\text{LR}}(\mathcal{G}_h, \mathcal{D})$ can learn the correct equivalence class when the sample size $n \rightarrow \infty$ and the approximation is good enough to preserve localized differences of score. 

\noindent \textbf{Discussions.} As analyzed earlier, the GES algorithm's advantage lies in its theoretical guarantee of consistency when paired with our score function, while our acceleration method addresses the computational challenge associated with large sample sizes, ensuring the framework's effectiveness and stability across various settings. However, GES faces the scalability issue when the number of graph nodes increases, because the search algorithm becomes computationally demanding. To tackle this issue, various strategies~\citep{ramsey2017million, zhang2024towards} have been proposed, in which leveraging modern continuous optimization techniques~\cite{Notears} to enhance scalability is a promising direction. Exploring more efficient search algorithms that integrate seamlessly with our score function represents a promising avenue for future research, potentially addressing these challenges.

\section{Performance Evaluation}
In this section, we first describe the experimental settings, then present the experimental results of sampling parameter selection, run-time reduction, and causal discovery performance on both synthetic and real-world data.  
\label{sec:experimental}
\subsection{Experimental Settings}
\textbf{Compared methods.} Our method is compared with the following existing methods. For constraint-based methods, we choose the PC algorithm~\cite{spirtes2001causation} and the max-min Markov blanket (MM-MB) search with symmetry correction~\cite{tsamardinos2003time} as they are widely used baselines.
We use KCIT~\cite{zhang2012kernel} to evaluate conditional independence in both algorithms since it is a non-parametric test that enables handling nonlinear causal relationships and various distributions. 
In the sequel, we denote these two methods as PC and MM, respectively. For score-based methods, we implement GES~\cite{chickering2002optimal} with four different score functions include BIC~\cite{schwarz1978estimating}, BDeu~\cite{buntine1991theory}, SC~\cite{sokolova2014causal} and CV likelihood score. BIC score can only be applied for continuous data, while BDeu score can only be applied for discrete data. SC score, which uses Spearman rank correlation to approximate mutual information, is unsuitable for multi-dimensional data. CV likelihood score introduced in Sec.~\ref{sec:generalizedscore} can be applied to all the evaluation settings. In the subsequent experiments, they are denoted as BIC, BDeu, SC and CV, respectively. For BDeu score, the equivalent sample size is set to $n'=1$. For MM and PC, we set the median distance for the kernel width of KCI, and the significance level as $0.05$ for (conditional) independence test. For CV likelihood score, we adhere to its recommended setting, employing $10$-fold cross-validation and setting the kernel width to twice the median distance. In addition to these methods, we also provide in the Appendices the results of some other compared methods that in some cases perform poorly and therefore are not included in the evaluation of the main experiments.

\noindent\textbf{Setting of our method.}
Our method is a low-rank approximation of CV, so denoted by CV-LR. In experiments, for low-rank approximation, we set the maximal rank parameter $m$ to $100$, i.e, the number of pivots for sampling is no more than $100$. The reason for this choice will be explained in Sec.~\ref{subsec:samplingparameterchoice}. Additionally, as we mention in Sec.~\ref{sec:generalizedlowrank}, for the discrete case that the number of distinct values is less than $m$, Alg.~\ref{alg2} is used. And for the remaining cases, Alg.~\ref{alg1} is used with the precision parameter set as $10^{-6}$. For the other parameters, all align with those used in CV. 

\noindent \textbf{Implementation Details.}
For score-based methods, we employ the GES algorithm, leveraging the implementation provided in the \texttt{causal-learn} package~\footnote{\url{https://github.com/py-why/causal-learn}}. For the other methods, we use their original implementations if available. For those without code, such as SC and MM-MB, we implement them from scratch. All methods are tested on the same datasets and hardware for fairness. The experiments are conducted on a system with an Intel Core i9-10900T processor (10 cores, 20 threads) and an NVIDIA RTX 3080 GPU. More details, including the experimental code, are available at \url{https://github.com/zhc-hydrion/2025KDD_CV_LR}.

\noindent\textbf{Accuracy measurements.} We consider two accuracy measurements. (1) F1 score that combines precision and recall into a single metric, is employed to measure the accuracy of the recovered causal skeleton. (2) The normalized structural hamming distance (SHD)~\cite{tsamardinos2006max}, which assesses the accuracy of the recovered causal directions by calculating the difference between the recovered Markov equivalence class and the true one.

\subsection{Sampling Parameter Choice}
\label{subsec:samplingparameterchoice}
Here, we determine the sampling parameter for the following experiments. We use two types of data: continuous data generated according to the process described in Sec.~\ref{syntheticdata} and discrete data sampled from the CHILD dataset, as outlined in Sec.~\ref{benchmarkdatasets}. For each data type, we explore two scenarios: one with an empty conditional variable set and another where the conditional variable set contains $6$ variables. We select a node as $X$ and (optional) choose $6$ variables as the conditional set, then evaluate both CV and CV-LR scores. We adjust $m$ to ensure the relative approximation error is no more than $0.5\%$, and find that $m=100$ meets this requirement (see the Appendix~\ref{appendix: sub: Sampling Parameter Choice} for more details). Therefore, $m=100$ is used for all subsequent experiments. Note that since the low-rank algorithm is adaptive so the actual number of samplings does not necessarily reach $m$, i.e., $m=100$ serves as an upper bound. It may be much smaller than 100 when the structure of the data is simpler.

\begin{figure}
\centering
\includegraphics[width=8.4cm]{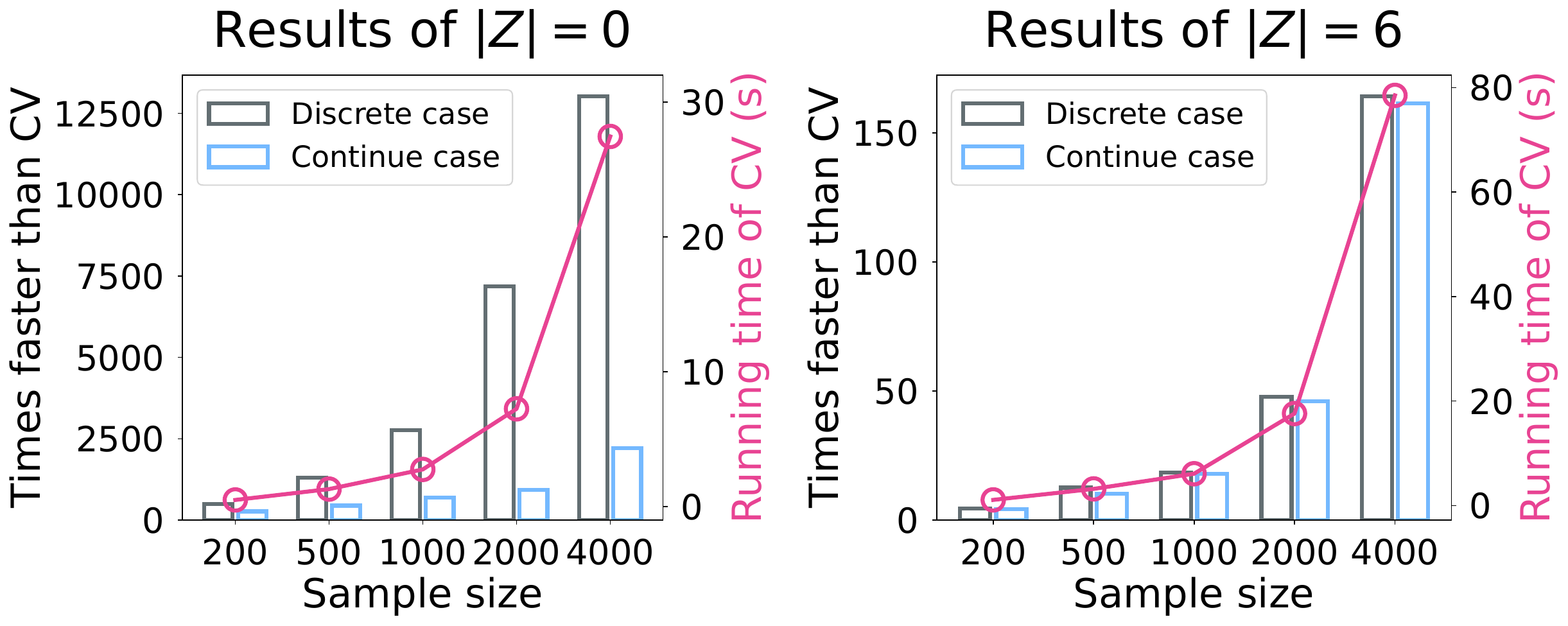}
\caption{The run-time results of CV and CV-LR under various sample size settings. In each subfigure, the results of the bar graph correspond to the left y-axis and the results of the curve correspond to the right y-axis.}
\label{fig:runningtime}
\end{figure}

\subsection{Results of Computation Speedup}
As our method is a fast version of the CV likelihood score, here we compare them in terms of computation efficiency. 
Fig.~\ref{fig:runningtime} shows the run-time results (for a single calculation of score) of CV and CV-LR under similar settings in Sec.~\ref{subsec:samplingparameterchoice}.
Overall, the relative difference in speed between CV-LR and CV rapidly grows with sample size. Comparing the results for 
$|Z|=0$ and $|Z|=6$, we observe that CV-LR demonstrates significant speedup when the conditional set is empty. This is because simpler data, with a lower rank of the kernel matrix, allows the algorithm to terminate earlier, resulting in fewer samplings. A similar pattern is obvious in both discrete and continuous settings for $|Z|=0$, with the discrete case showing a notably larger speedup ratio. This boosted speedup can be attributed to the specialized algorithm design in Sec.~\ref{sec:generalizedlowrank}. When $|Z|=6$, where the conditional set is larger and the structure of data more complex, the algorithm requires more samplings. Overall, CV-LR achieves significant speedup while maintaining accuracy, even with small sample size ($200$). In larger sample setting ($n=4000$), our method achieves speedup of over 150x, and in cases where the data structure is fully utilized (e.g. 
$|Z|=0$), it achieves speedup of 2000x in continuous data and 10,000x in discrete data.

\subsection{Performance on Synthetic Data}
\label{syntheticdata}
\textbf{Data generation process.} To demonstrate the causal discovery 
 effectiveness and versatility of the proposed method across various scenarios, we generate several types of data. These include continuous data, where all variables in the causal graph are continuous; mixed continuous and discrete data, where some variables are continuous and others are discrete with a $0.5$ probability; and multi-dimensional data, where variables have dimensions ranging from $1$ to $5$. The data for each variable 
$X_i$ is created using the following functional causal model:
\begin{equation}
X_i = g_i(f_i(\text{Pa}_i)+\epsilon_i)
\end{equation}
where $f_i$ denotes the causal mechanism, selected randomly from among linear functions, sine functions, tanh functions, logarithmic functions, or their combinations. $g_i$ represents the post-nonlinear distortion applied to variable $X_i$, which is chosen from linear, exponential, and power functions with power value ranging from $1$ to $3$. Meanwhile, $\epsilon_i$ is the noise term, which is randomly selected from Gaussian and uniform distributions. 

We create causal structures with varying graph densities ranging from $0.2$ to $0.8$, where the density is defined as the ratio of the number of edges to the maximum possible
number of edges in the graph. Each graph includes $7$ variables. Additionally, we generate the data with different sample sizes of $n$=200, 500 and 1000. For every combination of graph density, data type, and sample size, we generate $20$ separate realizations. The results with sample size $200$, $1000$ are shown in Fig.~\ref{fig:synthetic_200}, Fig.~\ref{fig:synthetic_500}, Fig.~\ref{fig:synthetic_1000} respectively. 
\begin{figure}
\centering
\includegraphics[width=8.3cm]{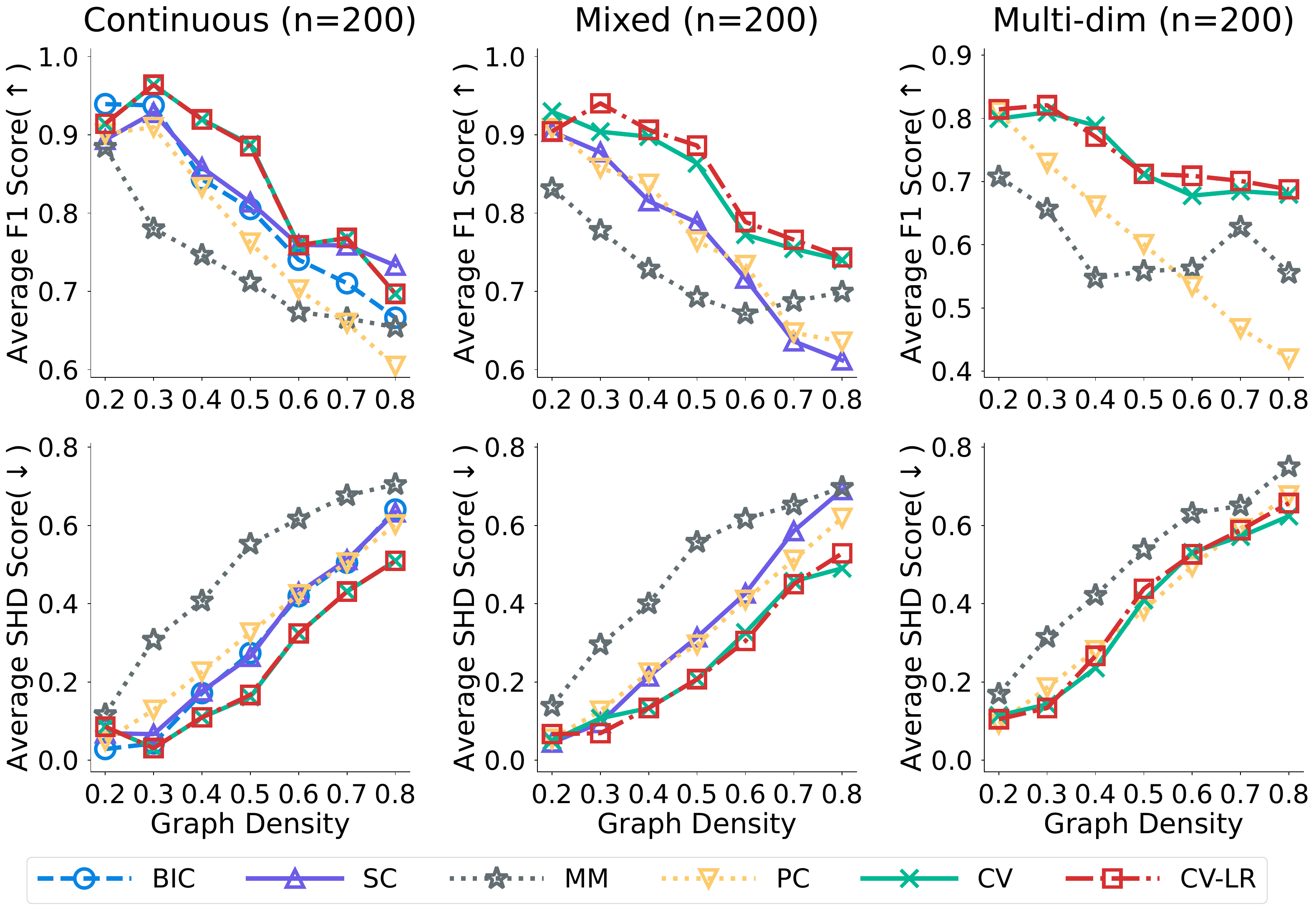}
\caption{The F1/SHD score of recovered causal graphs with sample size $n=200$. Left: continuous data. Middle: mixed continuous and discrete data. Right: multi-dimensional data.}
\label{fig:synthetic_200}
\end{figure}

\begin{figure}
\centering
\includegraphics[width=8.1cm]{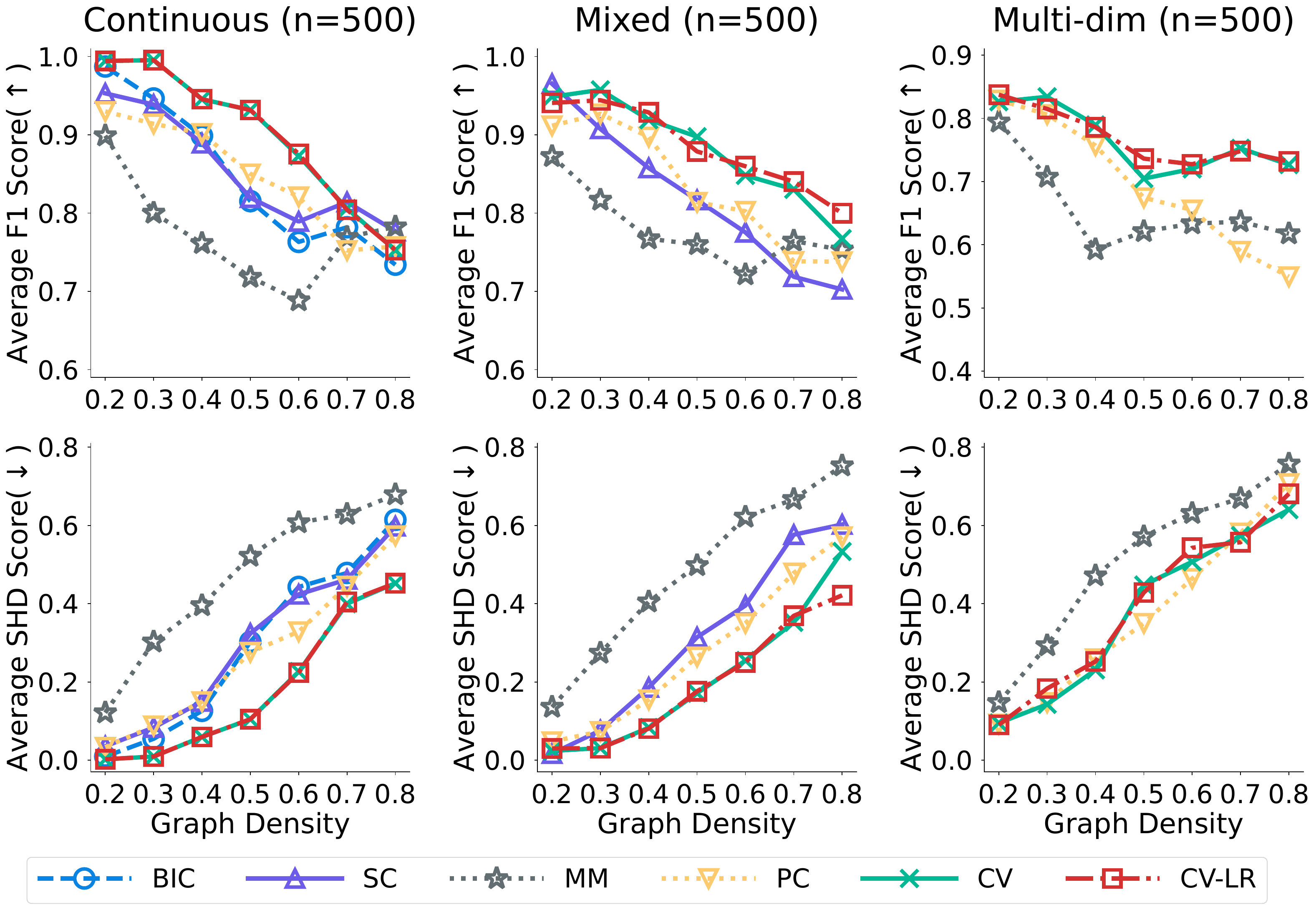}
\caption{The F1/SHD score of recovered causal graphs with sample size $n=500$. Left: continuous data. Middle: mixed continuous and discrete data. Right: multi-dimensional data.}
\label{fig:synthetic_500}
\end{figure}

\noindent\textbf{Results analysis.} By comparing CV-LR with CV, we observe that CV-LR's results are generally consistent with those of CV across most scenarios, showing only minor fluctuation. Specifically, CV-LR achieves nearly identical outcome to CV in continuous setting, demonstrating its effectiveness in this context. In the other two settings where mixed variables and multidimensional variables introduce additional complexity, small deviation occurs due to the limited sampling parameter $m=100$. However, CV-LR still maintains a valid approximation overall. Regarding accuracy, CV-LR and CV achieve the best results in most cases, particularly in high-density multidimensional data setting. In contrast, constraint-based methods cannot handle high-density multidimensional setting well, this may be because the power of conditional independence tests decreases as the conditioning set grows with graph density. These findings highlight the robustness of our approximation across various data types and its effectiveness across different sample sizes.

\subsection{Performance on Real-world Datasets}
\label{benchmarkdatasets}
Here, we evaluate the causal discovery performance of the proposed method on two benchmark discrete networks, the SACHS network and CHILD network, where all variables are discrete with cardinality ranging from $1$ to $6$. The SACHS network comprises $11$ variables and $17$ edges, while the CHILD network consists of $20$ variables with $25$ edges. Therefore, both networks are relatively sparse in nature. For each network, we randomly choose data points with sample size $n$ = 200, 500, 1000 and 2000 and repeat $20$ times in each case. The results are shown in Fig.~\ref{fig:real_ex}. 

\noindent\textbf{Result analysis.} By comparing CV-LR to CV, we find that both methods yield nearly identical accuracy results, but CV-LR significantly outperforms CV in terms of runtime. Specifically, CV requires approximately 2 hours and 4 hours for the two datasets, whereas CV-LR completes the same tasks in just about 5 seconds and 20 seconds, showing a speedup of over 1000x and 600x, respectively. This highlights the superior performance of CV-LR on discrete data, particularly in sparse network structures. Regarding causal discovery accuracy, CV and CV-LR achieve the highest accuracy on the SACHS dataset across all sample sizes, while BDeu score performs better on the CHILD dataset. Overall, F1 score gets improved with sample size for most of methods, indicating that accuracy increases with sample size on both the SACHS and CHILD datasets.

\begin{figure}
\centering
\includegraphics[width=8.3cm]{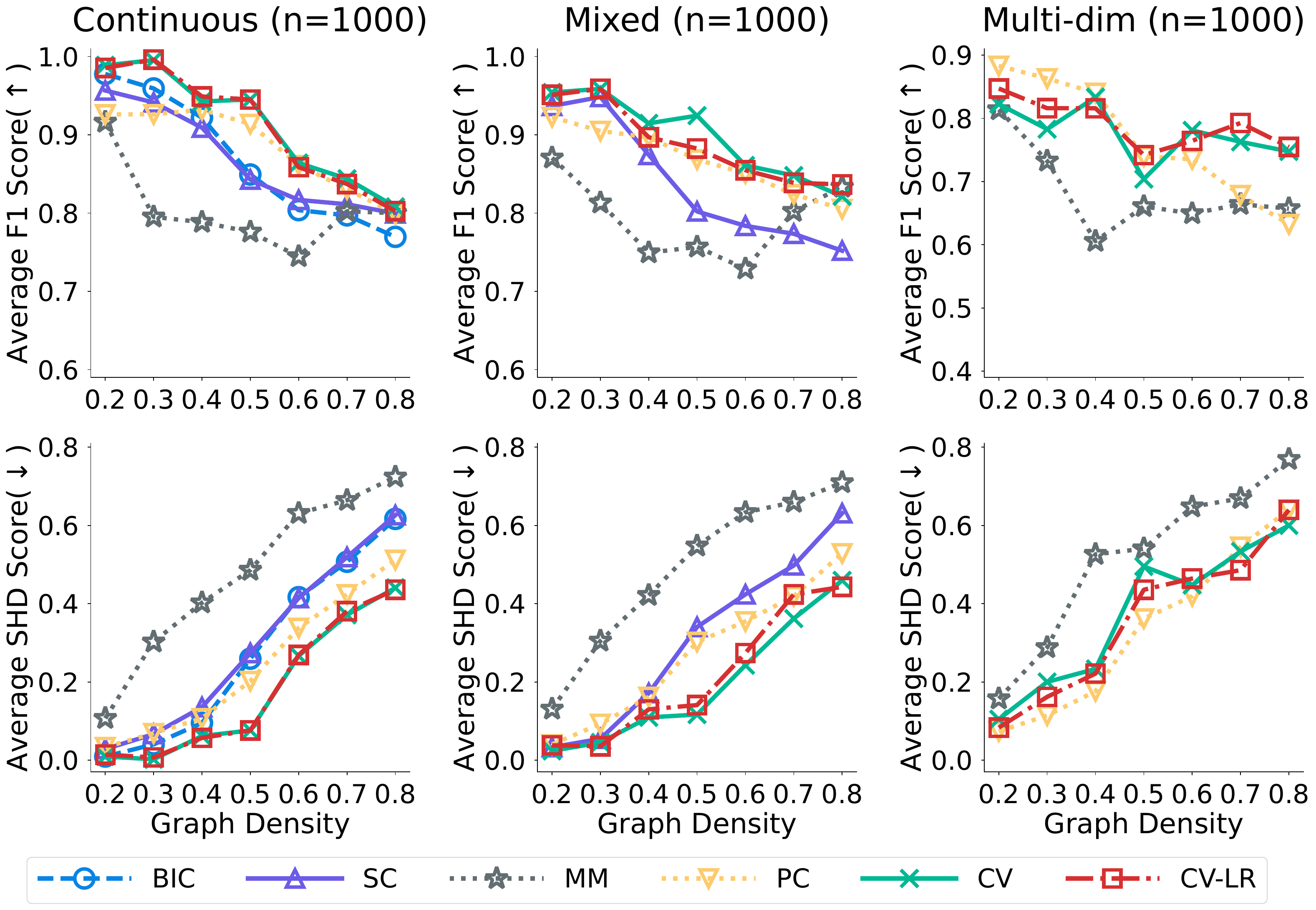}
\caption{The F1/SHD score of recovered causal graphs with sample size $n=1000$. Left: continuous data. Middle: mixed continuous and discrete data. Right: multi-dimensional data.}
\label{fig:synthetic_1000}
\end{figure}

\begin{figure}
\centering
\includegraphics[width=8.4cm]{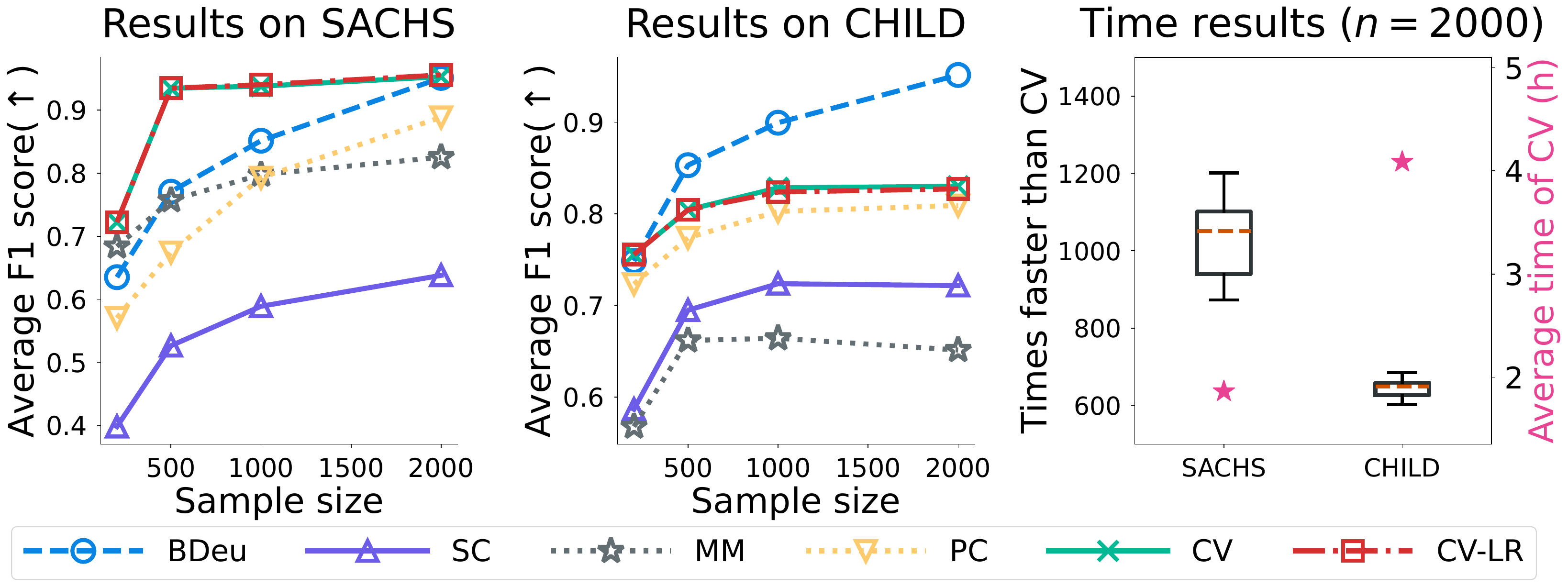}
\caption{Left two: F1 score results on the two discrete networks. Right: run-time results of CV and CV-LR on two discrete networks with sample size $n=2000$.  Results marked with pentagrams correspond to the right y-axis.}
\label{fig:real_ex}
\end{figure}

\section{Conclusion and Future Work}
\label{sec:conclusion}
This paper proposes a fast causal discovery method by approximate kernel-based generalized score function using a low-rank technique and a set of rules to manage the complex composite matrix operations required for score calculation, achieving $\mathcal{O}(n)$ time and space complexity. Additionally, we develop sampling algorithms tailored to different data types, enhancing efficiency in handling diverse datasets. Extensive experiments on both synthetic and real-world data demonstrate that our method significantly reduces computational costs compared to existing methods while maintaining comparable accuracy, particularly with large datasets. 
Future work will focus on boosting kernel selection efficiency within the score function and integrating the method with search algorithms capable of handling datasets with more variables.

\vspace{1\baselineskip}
\vspace{1\baselineskip}
\vspace{1\baselineskip}

\bibliographystyle{ACM-Reference-Format}
\bibliography{sample-base}

\appendix

\section{Details of Experiment Settings}
\label{appendix: exsettings}
In this appendix, we provide the details of synthetic dataset and the implementation details of compared methods.

\subsection{Details of Synthetic Data}
In the experiment on synthetic data, the generated graph consists of a total of $7$ variables. Based on the given graph density,
we randomly generate the corresponding number of edges in the graph. The leaf node $X_{leaf}$ in the graph, which has an
in-degree of $0$, follows either a normal distribution $\mathcal{N}(0, 1)$ or a uniform distribution $U(-0.5, 0.5)$ with equal probability. For
each response variable $X_i$ in the graph, its data is generated according to the following causal relation:
\begin{equation}
\label{pnl}
X_i = g_i(f_i(\text{Pa}_i)+\epsilon_i)
\end{equation}
where $\text{Pa}_i$ represents the parent nodes of $X_i$, and $f_i$ denotes the causal mechanism by which these parents affect the node. $f_i$ is selected with equal probability from $\{\text{linear}$, $\sin$, $\cos$, $\tanh$, $\log\}$, where the weight of the linear function ranges from $[0,1.5]$. $g_i$ is equally sampled from $\{\text{linear}, \exp, x^\alpha\}$, where the weight of the linear function ranges from $[1,2]$, and $\alpha$ in $x^\alpha$ is randomly selected from $\{1, 2, 3\}$. And $\epsilon_i$ is additive noise, which either follows a uniform distribution $U(-0.25, 0.25)$ or a normal distribution $\mathcal{N}(0, 0.5)$.

For mixed data that includes both continuous and discrete variables, we first generate the continuous data and then discretize it. We select 50\% of the variables for equal-frequency discretization, with the data values ranging from $1$ to $5$.

For multi-dimensional variables, each variable is randomly assigned a dimension, ranging from $1$ to $5$. The generation for leaf nodes follows the same method as before. For the remaining variables, take $X_i$ as an example, we first multiply its parent nodes $\text{Pa}_{i}$ by a matrix of all ones with the shape (dimension of $\text{Pa}_{i}$, dimension of $X_i$). This ensures that the parent nodes are transformed into a matrix that matches the $X_i$'s dimension. We then use Eq.~(\ref{pnl}) to generate the corresponding variable.

\subsection{Implementation Details}

For score-based methods, we employ the GES algorithm, leveraging the implementation provided in the \texttt{causal-learn} package~\footnote{\url{https://github.com/py-why/causal-learn}}. For a fair comparison, all methods are evaluated with the same data instances and hardware environments. The experiments are conducted on a system with an Intel Core i9-10900T processor (10 cores, 20 threads) and an NVIDIA RTX 3080 GPU. 

\noindent \textbf{Baselines. }
We use most existing implementations for the compared methods except the SC and MM-MB. For SC, MM-MB, since no official code exists, we implement them by ourselves. 
\begin{itemize}[leftmargin=5mm]
\item \textbf{SC}: SC~\cite{sokolova2014causal} is a classic score that can handle mixtures of discrete and continuous variables in Bayesian network structure learning. We have implemented an adaptation of SC that approximates mutual information by substituting Pearson correlation with Spearman correlation in the BIC score calculation, thus capturing monotonic relationships between variables.
\item \textbf{MM}: MM-MB~\cite{tsamardinos2003time} is a well-established algorithm for identifying Markov blankets. To achieve global causal discovery, we extend the Markov blanket identification from the target variable to all variables globally and use conditional independence tests to distinguish its spouse nodes, thereby uncovering the complete causal graph. To handle linear and nonlinear causal relationships as well as data from various distributions, we use KCI~\cite{zhang2012kernel} for the CI tests. Additionally, we apply symmetry correction to enhance its performance.
\item  \textbf{PC}: PC~\cite{spirtes2001causation} is a widely-used search algorithm. And we use the kernel-based conditional independence,
KCI~\cite{zhang2012kernel}, to test conditional independence relationships. The KCI test can handle nonlinear causal
relations and data from arbitrary distributions. The implementations of PC and KCI are based on causal-learn package.
We select a significance level of $\alpha = 0.05$. The median heuristic for the kernel bandwidth are used. 
\item \textbf{CV}: CV likelihood~\cite{huang2018generalized} is based on the RKHS regression model and uses cross-validation to avoid overfitting. Its code is provided in causal-learn package. We use the default setting with $10$-fold cross validation and set the regularization parameter $\lambda = 0.01$. Also, we set the parameter $\gamma = 0.01$ that used to ensure the positive-definite of the estimate covariance matrix.
\end{itemize}

\noindent \textbf{The code} for the experiments is given in the link~\url{https://github.com/zhc-hydrion/2025KDD_CV_LR}.

\section{Additional Experiment Results} 
\label{appendix: addiexpe}
In this appendix, we provide the details of experiment results in Sec.~\ref{subsec:samplingparameterchoice}, and the results on the synthetic data of size $500$. 

\subsection{Results on Sampling Parameter Choice}
\label{appendix: sub: Sampling Parameter Choice}
Here, we present detailed experimental results as discussed in Sec.~\ref{subsec:samplingparameterchoice}. In the main paper, we state the relative error is controlled to be no more than $0.5\%$ in all settings. The complete results are shown in Table.~\ref{tab:scores}. In cases where the data structure is simpler (both in the discrete case and in the continuous case with an empty condition set), the approximation performs very well, with a relative error less than $0.1\%$ across all sample sizes.

\begin{table}[h]
    \caption{The score results of CV and CV-LR under the settings in Sec.~\ref{subsec:samplingparameterchoice}. The maximal rank parameter $m$ is set to $100$. }
    \label{tab:scores}
    \centering
    \small
    \renewcommand\arraystretch{1.0}
    \scalebox{0.83}{
    \begin{tabular}{c|c|rrc}
        \toprule
        \textbf{Setting} & \textbf{$n$} & \textbf{CV score} & \textbf{CV-LR score} & \textbf{Relative error $(\%)$} \\
        \midrule
        \multirow{5}{1.0cm}{Continu. \\ $|Z|=0$} 
                   & 200 & 433.24125237132  & 433.24125162285 & <0.10 \\
                   & 500 & 2459.28117262525 & 2459.28117233407 & <0.10 \\
                   & 1000 & 9504.85766520591 & 9509.65871874661 & <0.10\\
                   & 2000 & 37386.49855593080 & 37406.16087222520 & <0.10 \\
                   & 4000 & 148271.47265144500 & 148322.40624804000 & <0.10 \\
        \midrule
        \multirow{5}{1.0cm}{Discrete \\ $|Z|=0$}
                   & 200 & 387.037158799239  & 387.03715879924  & <0.10 \\
                   & 500 & 2346.97607926954 & 2346.97607926954  & <0.10 \\
                   & 1000 & 9288.81581541630 & 9288.81581541630  & <0.10 \\
                   & 2000 & 36956.57345472810 & 36956.57345472810 & <0.10 \\
                   & 4000 & 147424.72166874500 & 147424.72166874500 & <0.10 \\
        \midrule
        \multirow{5}{1.0cm}{Continu. \\ $|Z|=6$}
                   & 200 & 437.94140907312 & 435.78203964329 & 0.50 \\
                   & 500 & 2449.31255885416 & 2446.87857237056 & 0.10 \\
                   & 1000 & 9476.68972889562 & 9489.42354423989 & 0.13 \\
                   & 2000 & 37320.18405900620 & 37390.72738080120 & 0.19 \\
                   & 4000 & 148134.19732984900 & 148319.96250081400 & 0.13 \\
        \midrule
          \multirow{5}{1.0cm}{Discrete \\ $|Z|=6$}
                   & 200 & 385.43869263321 & 385.37209674137 & <0.10 \\
                   & 500 & 2339.06660923831  & 2339.01611428266 & <0.10 \\
                   & 1000 & 9273.87967988472  & 9273.85932786237 & <0.10 \\
                   & 2000 & 36923.38881587080 & 36923.58768460090 & <0.10 \\
                   & 4000 & 147356.13710583600 & 147356.35657453700 & <0.10 \\
        \bottomrule
    \end{tabular}}
\end{table}

\subsection{Comparison with More Approaches}
In this section, we present results by including additional comparison methods, primarily focusing on continuous optimization-based approaches. Specifically, we evaluate GraN-DAG~\cite{GraN-DAG}, NOTEARS~\cite{Notears}, DAGMA~\cite{DAGMA}, and SCORE~\cite{SCORE}.  GraN-DAG, NOTEARS, and DAGMA utilize differentiable optimization frameworks to infer directed acyclic graphs (DAGs) from continuous data. They achieve this by employing functions to enforce acyclicity and optimizing the adjacency matrix of graphs to minimize their respective score criteria. In contrast, SCORE, a score-based method, leverages non-linear additive noise models and approximates the Jacobian of the data distribution's score function to recover causal structures. 

We conduct experiments on the SACHS dataset under the setting $n=2000$. The implementation details and hyperparameters for each comparison method are summarized below:
\begin{itemize}[leftmargin=5mm]
\item \textbf{NOTEARS}~\footnote{\url{https://github.com/xunzheng/notears}}: We adopted the same hyper-parameters recommended in the original paper.  Specifically, the graph threshold was set to 0.3. And we set \( h_{\text{tol}} = 1 \times 10^{-8} \), \( \rho_{\text{max}} = 1 \times 10^{16} \). The $l_1-$ and $l_2-$ penalty parameters were set to \( \lambda_1 = 0.01 \) and \( \lambda_2 = 0.01 \), respectively.
\item \textbf{DAGMA}~\footnote{\url{https://github.com/kevinsbello/dagma}}: We used the default parameter settings provided in the implementation.  The $l_1-$ and $l_2-$ regularization parameters were set to \(\lambda_1 = 0\) and \(\lambda_2 = 0.005\), respectively.
\item \textbf{GraN-DAG}~\footnote{\url{https://github.com/huawei-noah/trustworthyAI/tree/master/gcastle}}: The default settings were applied. The model was configured with 2 hidden layers, each containing 10 neurons, and used the "leaky-relu" activation function. The convergence threshold was set to \( h_{\text{threshold}} = 1 \times 10^{-7} \).
\item \textbf{SCORE}~\footnote{\url{https://github.com/paulrolland1307/SCORE}}: We used the default parameter settings.
\end{itemize}

\begin{table}[h!]
\centering
\caption{The results on the SACHS ($n=2000$). The SCORE method can not handle this setting, we marked as $-$ instead.}
\label{tab:discrete2000}
\begin{tabular}{l|cc}
\hline
\textbf{Method} & \textbf{F1 Score} ($~\uparrow~$) & \textbf{SHD Score} ($~\downarrow~$) \\ \hline
SCORE~\cite{SCORE}          & $-$            & $-$        \\
GraN-DAG~\cite{GraN-DAG}        & 0.27            & 0.25       \\
NOTEARS~\cite{Notears}         & 0.19            & 0.27       \\
DAGMA~\cite{DAGMA}           & 0.42            & 0.24       \\ 
\rowcolor{Gray} CV-LR (Ours)   & \textbf{0.94}              & \textbf{0.10}         \\
\hline
\end{tabular}
\end{table}

Each experiment was repeated 10 times, and the average results were reported. The results are presented in Tab.~\ref{tab:discrete2000}, demonstrating that our method significantly outperforms the comparison methods. In contrast, the comparison methods face notable challenges, resulting in poor performance. This can be attributed to their design, which primarily targets continuous data scenarios. As a result, the optimization procedures fail to converge effectively in discrete cases, leading to suboptimal outcomes. Similarly, SCORE struggles due to a mismatch between its underlying assumptions and the characteristics of the discretized data distribution. These results highlight the robustness and adaptability of our approach in handling various data distribution scenarios.

\subsection{Additional Results on More Datasets}
In the main paper, we give results on a discrete version of SACHS, and to further validate the generality of our method, we also evaluated several methods on the continuous version of the SACHS dataset, which consists of 853 samples in its original continuous form without any preprocessing or discretization. Apart from the dataset itself, all experimental settings remained consistent with those in the previous experiments. Each experiment was repeated $10$ times, and the average results were reported.

The average SHD scores for the different methods are summarized in Tab.~\ref{tab:shd_853}. Our method demonstrates a significant advantage on this version of the SACHS dataset as well, showcasing its applicability to various data types. Furthermore, comparing the results of CV and CV-LR reveals that our approximation method maintains excellent performance and achieves results comparable to CV.
\begin{table}[h!]
\centering
\caption{The SHD score results on the continuous version of the SACHS dataset. The sample size is $n = 853$. }
\label{tab:shd_853}
\begin{tabular}{l|c}
\hline
\textbf{Method} & \textbf{SHD Score} ($~\downarrow~$) \\ \hline
SCORE~\cite{SCORE}           & 0.2182      \\
GraN-DAG~\cite{GraN-DAG}        & 0.2727    \\
NOTEARS~\cite{Notears}         & 0.2364      \\
DAGMA~\cite{DAGMA}           & 0.3091      \\
PC~\cite{spirtes2001causation}              & 0.2182      \\ \hline
CV              & \textbf{0.1818}      \\
\rowcolor{Gray} CV-LR (Ours)          & \textbf{0.1818}      \\
\hline
\end{tabular}
\end{table}

In summary, these results further validate the generality of our approach, fulfilling the goal of acceleration without sacrificing accuracy in real-world scenarios, highlighting its ability to deliver superior performance in a variety of data scenarios.

\end{document}